\documentclass[10pt,twocolumn,letterpaper]{article}
\PassOptionsToPackage{table}{xcolor}
\usepackage[pagenumbers]{cvpr}
\usepackage[utf8]{inputenc}
\usepackage[T1]{fontenc}
\usepackage{textcomp}
\usepackage{microtype}
\usepackage{graphicx}
\usepackage{stfloats} 
\usepackage{subcaption}
\usepackage{booktabs}
\usepackage{wrapfig}
\usepackage{pifont}
\newcommand{\cmark}{\ding{51}}
\newcommand{\xmark}{\ding{55}}
\usepackage{enumitem}
\usepackage{makecell}
\usepackage{nicefrac}
\usepackage{mathtools}
\usepackage{amsthm}
\usepackage{algorithm}
\usepackage{algorithmic}

\renewcommand{\paragraph}[1]{\vspace{.5em}\noindent\textbf{#1.}}
\theoremstyle{plain}
\newtheorem{theorem}{Theorem}[section]
\newtheorem{proposition}[theorem]{Proposition}

\theoremstyle{definition}

\theoremstyle{remark}

\definecolor{cvprblue}{rgb}{0.21,0.49,0.74}
\usepackage[pagebackref,breaklinks,colorlinks,allcolors=cvprblue,pdfversion=1.7]{hyperref}
\title{Train Short, Inference Long: Training-Free Horizon Extension
for Autoregressive Video Generation}
\author{%
Jia Li$^{1,2,*}$\quad Xiaomeng Fu$^{3,*}$\quad Xurui Peng$^{2}$\quad Weifeng Chen$^{2}$\quad Youwei Zheng$^{2}$\\[3pt]
Tianyu Zhao$^{2}$\quad Jiexi Wang$^{2}$\quad Fangmin Chen$^{2}$\quad Xing Wang$^{2,\dagger}$\quad Hayden Kwok-Hay So$^{1,\dagger}$\\[6pt]
{\normalsize $^{1}$The University of Hong Kong\quad $^{2}$ByteDance}\\
{\normalsize $^{3}$Institute of Information Engineering, Chinese Academy of Sciences}
}
\hypersetup{pdftitle={Train Short, Inference Long: Training-Free Horizon Extension for Autoregressive Video Generation},
pdfauthor={Jia Li, Xiaomeng Fu, Xurui Peng, Weifeng Chen, Youwei Zheng, Tianyu Zhao, Jiexi Wang, Fangmin Chen, Xing Wang, Hayden Kwok-Hay So}}
\ifdefined\pdfsuppressptexinfo\pdfsuppressptexinfo=-1\fi
\makeatletter
\patchcmd{\@maketitle}{\vskip .375in}{\vskip .1in}{}{}
\patchcmd{\@maketitle}{\vspace*{24pt}}{\vspace*{4pt}}{}{}
\patchcmd{\@maketitle}{\vspace*{12pt}}{\vspace*{6pt}}{}{}
\makeatother
\begin{document}
\maketitle
\begingroup
\renewcommand{\thefootnote}{*}
\footnotetext[0]{Equal contribution.\quad $^{\dagger}$Corresponding authors.}
\endgroup
\begin{abstract}
Autoregressive (AR) video diffusion models such as CausVid and Self Forcing enable efficient long-video generation through chunk-wise rollout with local sliding-window attention, yet suffer from visual degradation, appearance drift, and motion collapse beyond their training horizons.
From a long-context sequence modeling perspective, we identify two underexplored inference-time bottlenecks in AR video length extrapolation.
First, scarce high-quality long-video data and high training costs restrict video models to short training horizons. Consequently, low-frequency dimensions of 3D temporal RoPE remain under-exposed during training, inducing stronger spectral bias than in LLMs and hindering direct transfer of LLM context-extension methods.
Second, both standard independent noise initialization and consistency-oriented methods with positive inter-frame correlation lack high-frequency dynamic priors to counter motion decay over long rollouts.
To address these issues, we propose \textbf{FLEX} (\textbf{F}requency-aware \textbf{L}ength \textbf{EX}tension), a training-free, plug-and-play inference framework comprising three complementary components.
Frequency-aware 3D RoPE modulation interpolates under-exposed low frequencies while extrapolating well-exposed high frequencies.
Antiphase Noise Sampling (ANS) introduces negative intra-chunk correlation to inject high-frequency motion priors at initialization.
Meanwhile, an inference-only attention sink works with RoPE modulation to anchor global consistency.
Experiments on Self Forcing and CausVid show substantial gains in long-horizon synthesis, with Self Forcing + FLEX matching or surpassing both training-free and training-based baselines on VBench-Long at $6\times$ to $12\times$ length extrapolation.
Additional evaluations on Krea-14B and minute-level generation with LongLive further demonstrate strong generalization across model scales, inference pipelines, and native long-video generators.
\end{abstract}

\section{Introduction}
Video diffusion models have progressed rapidly from short-clip synthesis toward coherent long-duration generation~\cite{ho2022videodiffusionmodels,blattmann2023stablevideodiffusionscaling,chen2025sanavideoefficientvideogeneration,kodaira2025streamditrealtimestreamingtexttovideo,song2025history}.
To overcome the quadratic complexity of full-sequence bidirectional attention, recent autoregressive (AR) video diffusion models~\cite{ai2025magi1autoregressivevideogeneration,chen2025skyreelsv2infinitelengthfilmgenerative,sun2025ardiffusionasynchronousvideogeneration}, such as CausVid~\cite{yin2025slow} and Self Forcing~\cite{huang2025self}, adopt causal chunk-wise generation with local sliding-window attention, enabling efficient and scalable long-horizon video synthesis in principle.

In practice, however, extending pretrained AR models beyond their training horizon leads to \emph{long-horizon degradation}: visual fidelity decays, subject identity and scene attributes drift, and motion collapses into near-static frames or repetitive loops.
Training-based extensions such as LongLive~\cite{yang2025longlive} and Rolling Forcing~\cite{liu2025rolling} mitigate these issues through additional training, but incur substantial computational costs.
At inference time, retaining initial frames as attention sinks~\cite{xiao2023efficient} provides global context, yet naive reference-frame pinning can introduce abrupt scene transitions or repetitive playback.
Jointly preserving long-range consistency and motion dynamics without retraining thus remains an open research challenge for video models.

We study this problem from the perspective of long-context sequence modeling and identify two underexplored inference-time bottlenecks.
\textbf{(i) Frequency-wise exposure imbalance in 3D temporal RoPE.}
Unlike LLMs pretrained on abundant text spanning thousands of tokens, AR video models are typically trained on relatively short clips (e.g., $L_{\text{train}}=21$ latent frames, or roughly 5s) due to long-video scarcity and high spatiotemporal training costs, leaving most low-frequency rotary planes under-exposed within the training range and limiting positional extrapolation when relative distances exceed $L_{\text{train}}$.
This modality-specific imbalance may also explain why LLM context-extension methods such as Position Interpolation (PI)~\cite{chen2023extending} and YaRN~\cite{peng2023yarn} do not transfer directly to long-horizon AR video generation, as analyzed in detail in Appendix~\ref{sec:appendix_pi}.
\textbf{(ii) Lack of high-frequency dynamic priors in noise sampling.}
In diffusion models, initial noise can provide a prior for motion synthesis~\cite{chang2024warpednoise,burgert2025gowiththeflow}.
Existing methods such as FreeNoise~\cite{qiu2023freenoise}, FreeLong++~\cite{lu2025freelong++}, and LongVie~\cite{gao2025longvie} rely on noise rescheduling or positive inter-frame correlation to enhance consistency.
However, positive correlation concentrates spectral power at low temporal frequencies and may produce highly similar content over time, strengthening temporal continuity at the expense of motion and making consistency and dynamics difficult to sustain jointly over extended autoregressive rollouts.

To bridge short-clip training and long-horizon inference, we propose \textbf{FLEX} (\textbf{F}requency-aware \textbf{L}ength \textbf{EX}tension), a training-free, plug-and-play framework with three complementary designs.
(1) \textbf{Frequency-aware 3D RoPE Modulation} interpolates under-exposed low-frequency temporal dimensions and extrapolates well-exposed high-frequency dimensions~\cite{bloc97parts,peng2023yarn}, preserving multi-scale temporal discriminability over extended horizons.
(2) \textbf{Antiphase Noise Sampling (ANS)} reshapes the noise spectrum and shifts power toward higher temporal frequencies, providing a dynamic prior at initialization.
(3) \textbf{Inference-only Attention Sink} retains a few initial frames at their native temporal indices, operating with frequency-aware RoPE modulation to anchor global semantics over long autoregressive rollouts.

Our contributions are summarized as follows:
\begin{itemize}
\item From the perspective of long-context sequence modeling, we identify two underexplored inference-time bottlenecks in AR video length extrapolation: severe frequency-wise exposure imbalance in 3D temporal RoPE under short-clip training, and insufficient high-frequency dynamic priors in noise initialization during long-horizon inference.

\item We propose \textbf{FLEX}, a training-free, plug-and-play framework that combines Frequency-aware 3D RoPE Modulation, Antiphase Noise Sampling, and an inference-only Attention Sink to generate longer, high-quality videos with long-range consistency and motion dynamics.

\item Extensive experiments show that FLEX is compatible with AR video models such as CausVid and Self Forcing, matching or surpassing training-based and training-free counterparts at $6\times$ to $12\times$ length extrapolation. Additional results on Krea-14B and LongLive further demonstrate broad generalization across model scales, inference pipelines, and native long-video generation models.
\end{itemize}

Demo and code: \href{https://ga-lee.github.io/FLEX_demo/}{ga-lee.github.io/FLEX\_demo}.

\section{Related work}

\paragraph{Autoregressive Video Diffusion}
Recent video diffusion models replace full-sequence bidirectional attention~\cite{peebles2023scalablediffusionmodelstransformers,wan2025wan} with causal, chunk-wise autoregression~\cite{yin2025slow,huang2025self,chen2025skyreelsv2infinitelengthfilmgenerative,ai2025magi1autoregressivevideogeneration} to avoid quadratic spatiotemporal scaling.
CausVid~\cite{yin2025slow} distills bidirectional models such as Wan2.1~\cite{wan2025wan} into few-step causal generators via DMD~\cite{yin2024one,yin2024improved}, and Self Forcing~\cite{huang2025self} mitigates exposure bias through autoregressive self-rollout training.
Building on this paradigm, LongLive~\cite{yang2025longlive} applies Streaming Long Tuning with KV recaching on extended rollouts, while Rolling Forcing~\cite{liu2025rolling} jointly denoises frames across progressive noise levels.
Without retraining, Deep Forcing~\cite{yi2025deepforcingtrainingfreelong} extends Self Forcing via Deep Sink phase adjustment and participative KV compression.
Despite these advances, short-trained AR models still suffer from visual drift, temporal inconsistency, and motion decay when extrapolated to substantially longer generation horizons.

\paragraph{Long Context Extension in RoPE}
Rotary Positional Embedding (RoPE)~\cite{su2024roformer} encodes relative positions via orthogonal rotations in Transformers~\cite{touvron2023llama,yang2025qwen3}.
To handle sequences beyond pretraining lengths, LLM context-extension methods~\cite{zhong2024understandingropeextensionslongcontext} have progressed from linear Position Interpolation (PI)~\cite{chen2023extending} and NTK-aware scaling~\cite{bloc97ntk} to frequency-aware schemes such as NTK-by-parts~\cite{bloc97parts}, YaRN~\cite{peng2023yarn}, and Long\-RoPE~\cite{ding2024longropeextendingllmcontext}.
In contrast, context extension in AR video generation remains at a relatively early stage: existing efforts typically remap temporal indices to bound relative distances~\cite{yesiltepe2025infinityropeactioncontrollableinfinitevideo}, whereas the frequency-wise exposure imbalance of 3D RoPE under short-clip training remains largely underexplored.

\paragraph{Temporal Consistency and Noise Priors}
Maintaining long-range temporal coherence without sacrificing motion is a central challenge in video diffusion.
FreeNoise~\cite{qiu2023freenoise} stabilizes global structure in a training-free manner via noise rescheduling and window-based temporal attention fusion, while AnimateDiff~\cite{guo2023animatediff} learns transferable motion priors through a motion module to generate temporally smooth videos.
More recently, FreeLong++~\cite{lu2025freelong++} proposes SpecMix to blend shuffled base noise with independent residual noise in the spectral domain, and LongVie~\cite{gao2025longvie} reuses shared initial noise across video clips to strengthen inter-clip consistency.
Although these methods improve spatiotemporal smoothness, overly strong temporal correlation in initial noise may over-stabilize latent trajectories during extended autoregressive generation, thereby weakening motion dynamics.

\section{Method}

\begin{figure}[t]
  \centering
  \includegraphics[width=\columnwidth]{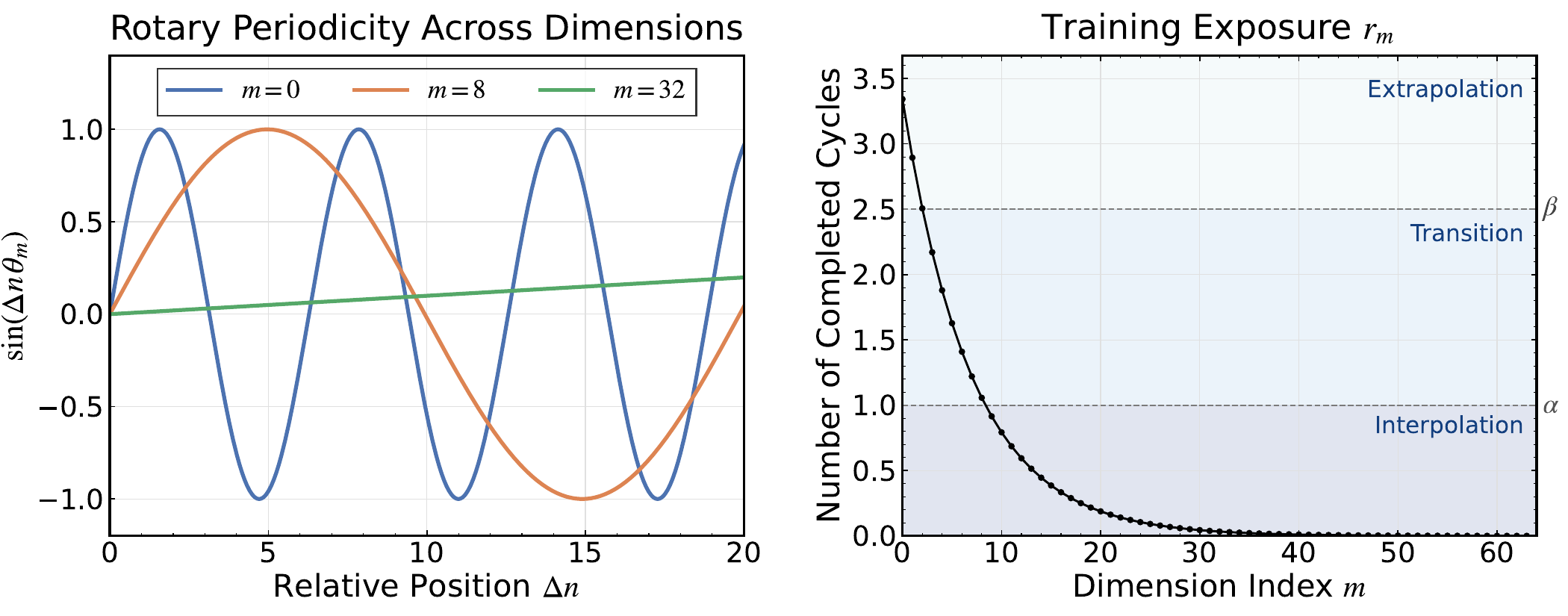}
\caption{\textbf{Frequency-aware analysis of temporal RoPE.} 
 (Left) Visualization of rotary periodicity $\sin(\Delta n \theta_m)$ along the temporal dimension. (Right) Training exposure $r_m$, defined as the number of completed cycles given a training horizon of $L_{\text{train}}=21$. }
  \label{fig:temporal_rope}
  \vspace{-0.8em}
\end{figure}

\subsection{Frequency-aware 3D RoPE modulation} 
\label{sec:ntk}
\label{sec:preliminary}

In chunk-wise AR video diffusion, an inference sequence of length $L$ is decomposed into $N=L/f$ successive chunks $\mathbf{C}=\{\mathcal{C}_i\}_{i=1}^N$, where each chunk $\mathcal{C}_i\in\mathbb{R}^{f\times H\times W\times C}$ contains $f$ latent frames. The $i$-th chunk is denoised from an initial Gaussian noise tensor $\mathbf{Z}_T^{(i)}$ (with the $f$ intra-chunk frames sampled i.i.d.\ in standard protocols) and conditioned on a historical context $\mathbf{\Phi}_{<i}$ maintained by a sliding local window of size $w$ implemented through KV caches. Video diffusion Transformers~\cite{wan2025wan} apply 3D RoPE across temporal, height, and width subspaces ($d_{\text{head}}=d_F+d_H+d_W$). Along the temporal subspace $d_F$, the encoding $f(\mathbf{q}, n, \theta)$ rotates the $m$-th coordinate pair ($m\in\{0,\dots,d_F/2-1\}$) at frame index $n$ by angle $n\theta_m$ with $\theta_m = b^{-2m/d_F}$ ($b=10000$), so attention scores depend on relative temporal distance $\Delta n = n_q - n_k$, as described in detail in Appendix~\ref{sec:appendix_prelim}.

\paragraph{Problem Definition: Training-to-Inference Horizon Gap} We define the temporal rotation wavelength at plane index $m$ as $\lambda_m = 2\pi / \theta_m = 2\pi \cdot b^{2m/d_F}$. While RoPE is theoretically shift-invariant, its empirical effectiveness is bounded by the relative phases encountered during training, $\mathcal{D}_{\Delta} = \{ \Delta n \theta_m \mid \Delta n \in [0, L_{\text{train}}-1], m \in \{0, \dots, d_F/2-1\} \}$. The local window bounds distances between recent frames, but the distance from retained sink frames to the current chunk grows with generation length (see Section~\ref{sec:sink}). When relative distances exceed $L_{\text{train}}$, differences in training exposure between the high- and low-frequency planes of temporal RoPE affect their extrapolation performance: (1) \textbf{High-Frequency Dimensions ($\lambda_m \ll L_{\text{train}}$):} These components undergo multiple full rotation cycles during training, yielding high \textit{training exposure} that supports stable generalization to larger relative distances $\Delta n > L_{\text{train}}$. (2) \textbf{Low-Frequency Dimensions ($\lambda_m > L_{\text{train}}$):} Due to their long wavelengths, these components fail to complete even a single $2\pi$ rotation within the training horizon, remaining \textit{under-exposed} during training and degrading when extrapolated to $\Delta n > L_{\text{train}}$.

Position Interpolation (PI)~\cite{chen2023extending} and YaRN~\cite{peng2023yarn} are widely used methods for extending LLM context windows. PI linearly rescales position indices as $n/S$, where $S = L/L_{\text{train}}$, to fit the extended sequence within the training range:
\begin{equation}
\label{eq:pi}
f_{\text{PI}}(\mathbf{q}, n, \theta) = f\left(\mathbf{q}, g(n) = \frac{n}{S}, h(\theta_m) = \theta_m\right).
\end{equation}
YaRN combines frequency-dependent RoPE interpolation with attention-logit scaling before softmax. Despite their effectiveness in language modeling, our qualitative experiments show that directly applying these methods does not yield stable length extrapolation in AR video generation: PI produces nearly static content with severe color oversaturation and rapid quality degradation, while YaRN exhibits texture drift and structural distortion, as shown in Appendix~\ref{sec:appendix_pi}.

One possible explanation for PI's motion collapse is the uniform compression of Fourier space. By scaling down all dimensions equally, PI reduces the phase difference between adjacent frames, which may weaken the model's ability to distinguish fine-grained temporal order. This potential loss of local temporal resolution motivates preserving high-frequency rotations, consistent with the local-distance analysis underlying NTK-by-parts scaling~\cite{bloc97parts,peng2023yarn}.

Building on this frequency-selective strategy from LLM context extension (e.g., Qwen3~\cite{yang2025qwen3}), we adapt frequency-aware 3D temporal RoPE modulation to long-horizon AR video generation. First, we quantify the phase coverage of the $m$-th rotary plane via \textit{training exposure} ($r_m$), defined as the number of rotation cycles completed within $L_{\text{train}}$:
\begin{equation}
    r_m = \frac{L_{\text{train}}}{\lambda_m} = \frac{L_{\text{train}} \cdot \theta_m}{2\pi}.
\end{equation}
Unlike LLMs trained on long text contexts, video diffusion models are typically restricted to short training clips (e.g., $L_{\text{train}}=21$ latent frames), resulting in limited temporal exposure. As shown in Fig.~\ref{fig:temporal_rope} (right), most dimensions exhibit $r_m < 1$ and fail to complete even a single rotation cycle during training.
Rather than compressing position indices as in Eq.~\eqref{eq:pi}, we keep the native temporal index $g(n)=n$ and apply dimension-wise frequency modulation $h(\theta_m)$:
\begin{equation}
    f_{\text{parts}}(\mathbf{q}, n, \theta) = f\bigl(\mathbf{q}, g(n)=n, h(\theta_m)\bigr).
\end{equation}
Here, $h(\theta_m)$ is governed by a piecewise linear ramp $\gamma(r_m)$ that partitions the rotary planes into three regimes according to their training exposure $r_m$ (schematically illustrated by the shaded regions in Fig.~\ref{fig:temporal_rope}):
\begin{equation}
\gamma(r_m) = 
\begin{cases} 
0, & \text{if } r_m < \alpha, \\
1, & \text{if } r_m > \beta, \\
\frac{r_m - \alpha}{\beta - \alpha}, & \text{otherwise}.
\end{cases}
\end{equation}
The modulated frequency $h(\theta_m)$ for each plane is computed as a convex combination:
\begin{equation}
    h(\theta_m) = (1 - \gamma(r_m)) \frac{\theta_m}{S} + \gamma(r_m) \theta_m.
\end{equation}

The horizon scaling factor is set by the target inference length $L$ relative to $L_{\text{train}}$ as $S = \max(1, L / L_{\text{train}})$. For $L \le L_{\text{train}}$, $S=1$ leaves the base model unchanged within its training horizon. For $L > L_{\text{train}}$, thresholds $\alpha$ and $\beta$ govern frequency-selective modulation across the temporal subspace $d_F$.
(i) \textbf{High-Frequency Extrapolation ($r_m > \beta$)} preserves native rotation frequencies to maintain fine-grained frame-to-frame temporal discriminability.
(ii) \textbf{Low-Frequency Interpolation ($r_m < \alpha$)} remaps under-exposed low-frequency components into the trained phase range to stabilize long-range positional encoding, with a linear transition across $\alpha \le r_m \le \beta$ between the two regimes.

\subsection{Antiphase noise sampling}
\label{sec:ans}

The initial noise state can guide motion synthesis in diffusion models~\cite{chen2023control}. Excessively strong positive temporal correlation may over-stabilize latent trajectories, reducing motion dynamics over long horizons. We therefore introduce Antiphase Noise Sampling (ANS), a chunk-wise initialization strategy that reshapes intra-chunk temporal covariance to inject high-frequency priors for motion dynamics.

\paragraph{Mathematical Formulation} For each chunk of length $f$ with flattened latent frames $\mathbf{Z} = [\mathbf{z}_0, \mathbf{z}_1, \dots, \mathbf{z}_{f-1}]^\top \in \mathbb{R}^{f \times d}$ ($d=HWC$), we sample independently across chunks and parameterize intra-chunk correlation via a first-order autoregressive (AR(1)) recursion:
\begin{align}
    \mathbf{z}_0 &\sim \mathcal{N}(\mathbf{0}, \mathbf{I}), \label{eq:ans_init} \\
    \mathbf{z}_u &= \rho \mathbf{z}_{u-1} + \sqrt{1-\rho^2} \boldsymbol{\epsilon}_u, \quad \boldsymbol{\epsilon}_u \sim \mathcal{N}(\mathbf{0}, \mathbf{I}), \label{eq:ans_recursion}
\end{align}
where $u \in \{1, \dots, f-1\}$ and $\rho \in [-1, 1]$ controls the within-chunk correlation across frames (yielding alternating signs $\mathbf{z}_u = (-1)^u \mathbf{z}_0$ at $\rho=-1$). Unlike consistency-oriented methods, we adopt the antiphase regime ($\rho < 0$), where adjacent noise frames exhibit negative covariance and elevated frame-to-frame variation at initialization.

\begin{figure}[t!] 
  \centering
  
  \includegraphics[width=\columnwidth]{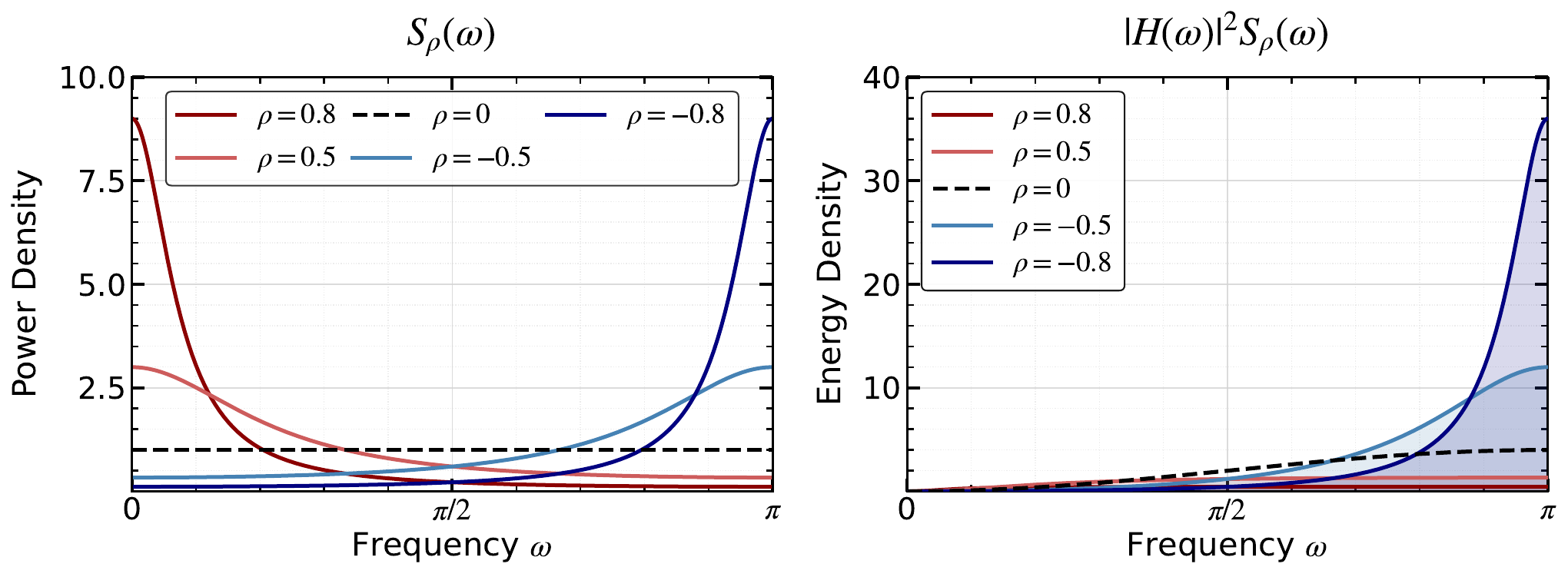}
     \caption{\textbf{Spectral analysis of Antiphase Noise Sampling.} (Left) Power Spectral Density $S_{\rho}(\omega)$. (Right) Motion Energy Density $|H(\omega)|^2 S_{\rho}(\omega)$. The antiphase regime ($\rho < 0$) shifts power toward the high-frequency passband, boosting motion energy relative to the i.i.d. baseline ($\rho=0$).}
  \label{fig:spectral_analysis}
  \vspace{-0.8em}
\end{figure}
\begin{proposition}[Distributional Invariance and Structure]
\label{prop:ans_properties}
Within each chunk, the ANS construction in Eqs.~\eqref{eq:ans_init}--\eqref{eq:ans_recursion} satisfies two properties:
(i) \textbf{Marginal Preservation:} each frame retains the standard Gaussian marginal $\mathbf{z}_u\sim\mathcal{N}(\mathbf{0},\mathbf{I})$.
(ii) \textbf{Toeplitz Covariance:} the intra-chunk cross-covariance follows $\operatorname{Cov}(\mathbf{z}_u,\mathbf{z}_v)=\rho^{|u-v|}\mathbf{I}$ (see Appendix~\ref{subsec:proof_energy}).
\end{proposition}

Preserving the unit-variance Gaussian marginal keeps the per-frame noise scale aligned with pretrained diffusion denoisers while modifying only the intra-chunk temporal covariance, enabling plug-and-play integration at inference time without changes to the model architecture.

\begin{table*}[!t]
\centering
\caption{VBench-Long quality evaluation on MovieGen prompts at 30s and 60s. All methods are adapted from Wan2.1-1.3B. Deep Forcing and FLEX are training-free extensions of Self Forcing. Bold marks the best displayed score in each duration group.}
\label{tab:main_comparison_moviegen}
\definecolor{flexrowblue}{RGB}{232,239,250}
\setlength{\tabcolsep}{5.5pt}
\setlength{\aboverulesep}{1pt}
\setlength{\belowrulesep}{2pt}
\resizebox{\textwidth}{!}{%
\begin{tabular}{l|ccccccc|c|c}
\toprule
Method & \makecell{Subject\\Consistency} & \makecell{Background\\Consistency} & \makecell{Motion\\Smoothness} & \makecell{Aesthetic\\Quality} & \makecell{Imaging\\Quality} & \makecell{Dynamic\\Degree} & \makecell{Temporal\\Flickering} & \makecell{Quality\\Score} & $\Delta_{\mathrm{Drift}}$ \\
\midrule
\multicolumn{10}{c}{30s (6$\times$ Extension)} \\
\midrule
LongLive~\cite{yang2025longlive} & 97.57 & 96.52 & \textbf{98.85} & 62.34 & 69.26 & 39.16 & 99.08 & 82.78 & $-1.32$ \\
Rolling Forcing~\cite{liu2025rolling} & 97.98 & 96.83 & 98.80 & 61.20 & \textbf{70.98} & 32.80 & 98.65 & 82.31 & $-0.25$ \\
Deep Forcing~\cite{yi2025deepforcingtrainingfreelong} & 96.66 & 95.85 & 98.05 & 59.81 & 67.32 & \textbf{57.27} & 98.53 & 82.53 & $-2.65$ \\
\midrule
CausVid~\cite{yin2025slow} & 98.28 & 97.03 & 98.27 & 60.03 & 65.57 & 32.65 & 98.49 & 81.04 & $-3.34$ \\
\rowcolor{flexrowblue}
CausVid + FLEX & \textbf{98.29} & \textbf{97.43} & 98.45 & 62.03 & 63.57 & 43.24 & 98.34 & 81.96 & $-0.65$ \\
\midrule
Self Forcing~\cite{huang2025self} & 98.07 & 96.78 & 98.45 & 60.13 & 69.57 & 26.24 & 98.58 & 81.21 & $-1.78$ \\
\rowcolor{flexrowblue}
Self Forcing + FLEX & 97.74 & 96.83 & 98.42 & \textbf{63.09} & 69.87 & 40.84 & \textbf{99.13} & \textbf{83.01} & $\mathbf{-0.06}$ \\
\midrule
\multicolumn{10}{c}{60s (12$\times$ Extension)} \\
\midrule
LongLive~\cite{yang2025longlive} & 97.30 & 96.37 & \textbf{98.83} & 62.18 & 69.13 & 39.98 & 99.01 & 82.68 & $-0.71$ \\
Rolling Forcing~\cite{liu2025rolling} & 97.86 & 96.73 & 98.79 & 60.43 & \textbf{70.89} & 32.48 & 98.62 & 82.10 & $-0.54$ \\
Deep Forcing~\cite{yi2025deepforcingtrainingfreelong} & 97.53 & 96.55 & 98.37 & 59.86 & 68.19 & \textbf{44.54} & 98.53 & 82.16 & $-0.75$ \\
\midrule
CausVid~\cite{yin2025slow} & 98.38 & 97.03 & 98.46 & 60.60 & 66.93 & 22.11 & 98.70 & 80.73 & $-0.06$ \\
\rowcolor{flexrowblue}
CausVid + FLEX & \textbf{98.44} & \textbf{97.13} & 97.98 & 60.41 & 67.85 & 40.44 & 98.65 & 82.01 & $\mathbf{+1.63}$ \\
\midrule
Self Forcing~\cite{huang2025self} & 97.74 & 96.64 & 98.34 & 57.67 & 68.82 & 25.57 & 98.55 & 80.51 & $-2.51$ \\
\rowcolor{flexrowblue}
Self Forcing + FLEX & 97.85 & 96.87 & 98.41 & \textbf{62.71} & 69.79 & 37.82 & \textbf{99.01} & \textbf{82.68} & $-0.56$ \\
\bottomrule
\end{tabular}%
}
\par\vspace{0.7em}
\centering
\caption{VBench-Long quality and semantic evaluation on 30s generation ($6\times$ extension) using the official VBench prompt set.}
\label{tab:main_comparison_vbench}
\definecolor{flexrowblue}{RGB}{232,239,250}
\setlength{\tabcolsep}{2pt}

\resizebox{\textwidth}{!}{
\begin{tabular}{lccccccc|c|c|c}
\toprule
Method
& \makecell{Subject\\Consistency}
& \makecell{Background\\Consistency}
& \makecell{Aesthetic\\Quality}
& \makecell{Dynamic\\Degree}
& \makecell{Object\\Class}
& \makecell{Spatial\\Relationship}
& \makecell{Overall\\Consistency}
& \makecell{\textbf{Quality}\\ \textbf{Score}}
& \makecell{\textbf{Semantic}\\ \textbf{Score}}
& \makecell{\textbf{Total}\\ \textbf{Score}} \\

\midrule
Self Forcing~\cite{huang2025self}
& 98.58 & 97.06 & 61.80 & 29.38 & 90.76 & 72.74
& 25.23 & 81.94 & 76.42 & 80.84 \\
LongLive~\cite{yang2025longlive}
& 98.36 & 97.22 & 64.13 & 39.84 & 94.44 & 79.24
& 26.43 & 83.39 & \textbf{80.85} & 82.88 \\
Rolling Forcing~\cite{liu2025rolling}
& 98.25 & 97.30 & 63.92 & 40.52 & 94.69 & 76.34
& 26.49 & 83.57 & 79.86 & 82.83 \\
Deep Forcing~\cite{yi2025deepforcingtrainingfreelong}
& 96.84 & 96.16 & 63.71 & \textbf{71.91} & 86.83 & 72.72
& 24.73 & \textbf{84.66} & 73.88 & 82.50 \\
\rowcolor{flexrowblue}
Self Forcing + FLEX 
& \textbf{98.70} & \textbf{97.94} & \textbf{65.43} & 44.32 & \textbf{94.94} & \textbf{82.04}
& \textbf{26.67} & 84.17 & 80.73 & \textbf{83.48} \\
\bottomrule
\end{tabular}
}
\vspace{-0.6em}
\end{table*}

\paragraph{Energy Monotonicity and Spectral Seeding} To quantify the impact of ANS on temporal variation, we analyze the expected \textit{adjacent-difference energy} across consecutive noise frames within each chunk $\mathcal{C}_i$:
\begin{equation}
\mathcal{E}_T(\rho)\triangleq\mathbb{E}\left[\sum_{u=1}^{f-1}\|\mathbf{z}_u-\mathbf{z}_{u-1}\|_2^2\right].
\end{equation}
\begin{proposition}[Energy Monotonicity]
\label{prop:energy}
For $\rho \in [-1, 1]$ under Eqs.~\eqref{eq:ans_init}--\eqref{eq:ans_recursion}, $\mathcal{E}_T(\rho)=2(f-1)(1-\rho)d$, which is strictly decreasing in $\rho$ and attains its maximum $4(f-1)d$ at $\rho=-1$, exceeding the i.i.d.\ baseline ($\rho=0$) throughout $\rho<0$ (the full proof is provided in Appendix~\ref{subsec:proof_energy}).
\end{proposition}

This temporal enhancement admits a frequency-domain interpretation. For a stationary AR(1) process with $|\rho|<1$, the Power Spectral Density (PSD) is given by
\begin{equation}
S_\rho(\omega) = \frac{1-\rho^2}{1 + \rho^2 - 2\rho \cos(\omega)}, \quad \omega \in [-\pi, \pi].
\label{eq:psd}
\end{equation}
As $\rho\to-1^+$, $S_\rho(\omega)$ develops a high-pass profile whose power concentrates toward the Nyquist frequency $\omega=\pi$ (Fig.~\ref{fig:spectral_analysis}, left), consistent with the alternating boundary $\mathbf{z}_u=(-1)^u\mathbf{z}_0$ at $\rho=-1$. Treating the temporal difference operator as a high-pass filter $H(\omega)=1-e^{-j\omega}$ with $|H(\omega)|^2=2(1-\cos\omega)$, the per-step, per-dimension difference energy $\bar{\mathcal{E}}_T(\rho) = \mathcal{E}_T(\rho)/[(f-1)d] = 2(1-\rho)$ integrates the filtered spectrum over $[-\pi,\pi]$ (see Appendix~\ref{subsec:proof_psd}). Because both $S_\rho(\omega)$ (for $\rho<0$) and $|H(\omega)|^2$ peak at $\omega=\pi$, the antiphase regime aligns noise power with the high-frequency passband of temporal differences.

As illustrated in Fig.~\ref{fig:spectral_analysis}, ANS reallocates initialization variance from low-frequency offsets toward rapid frame-to-frame fluctuations, counteracting the tendency of autoregressive rollouts to collapse into static frames over long horizons.

\subsection{Temporal attention sink}
\label{sec:sink}
While frequency-aware RoPE preserves multi-scale relative position encoding, sliding-window AR video models lose global context once early frames leave the KV cache. We retain the first $L_{\text{sink}}$ latent frames in the KV cache throughout generation as persistent global anchors. However, qualitative analyses in Appendix~\ref{sec:attn_sink} indicate that attention sinks used in isolation or in combination with YaRN remain insufficient to prevent color and structural drift. Unlike training-based sinks in LongLive~\cite{yang2025longlive} or phase-shifted sinks in Deep Forcing~\cite{yi2025deepforcingtrainingfreelong}, our inference-only anchor preserves native temporal indices without architectural modification or fine-tuning, operating jointly with frequency-aware temporal RoPE to support stable attention as sink-to-current distances grow.

\section{Experiments}
\label{sec:exp}

\begin{figure*}[!t]
\centering
    \includegraphics[width=\textwidth]{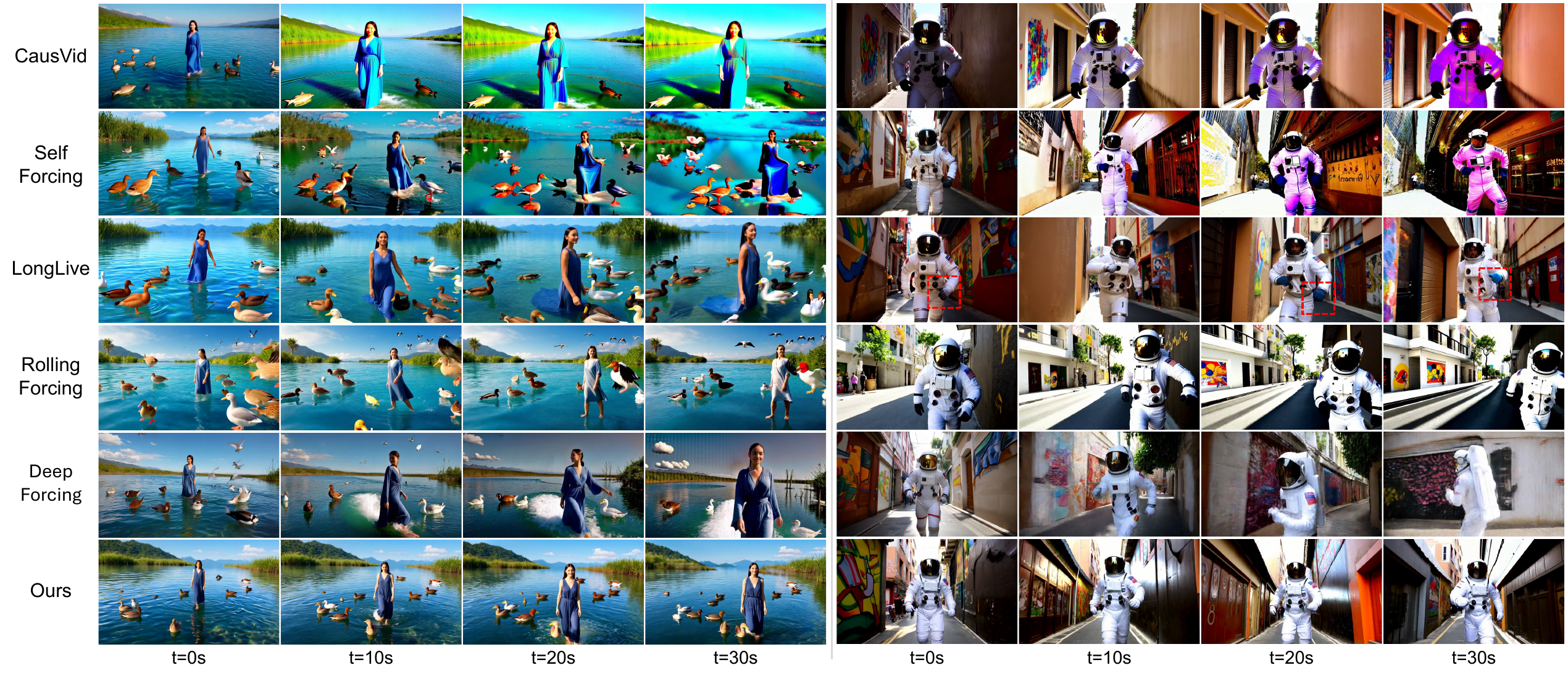}
    \caption{
    \textbf{Qualitative comparison at 30s.}
    We compare FLEX with five baselines. Our method yields higher-fidelity video generation with reduced artifacts and improved temporal consistency, maintaining consistent appearance attributes over time.
    }
    \label{fig:qualitative_30s}
    \par\vspace{0.4em}

\centering
    \includegraphics[width=\textwidth]{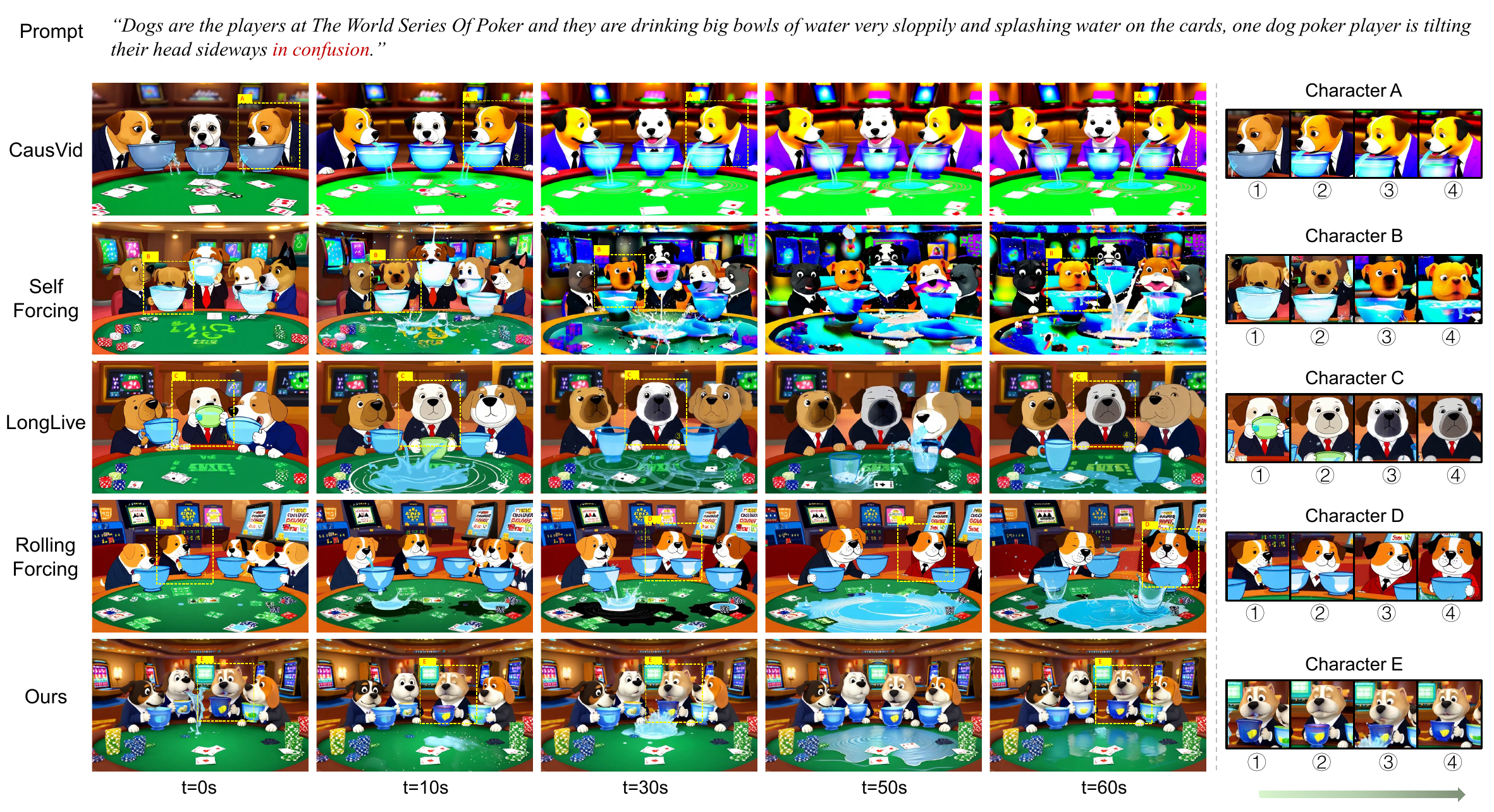}
    \caption{
    \textbf{Qualitative comparison at 60s.}
    Existing methods exhibit identity drift or scene collapse over long horizons, while our model maintains high subject consistency and visual quality. The right-hand panels track characters A--E, one from each method, highlighting FLEX's ability to preserve fine-grained character identity.
    }
    \label{fig:qualitative_60s}
    \vspace{-1.0em}
\end{figure*}

\paragraph{Implementation}
We adopt Self Forcing~\cite{huang2025self} with chunk size $f=3$ as the primary base model. Built on Wan2.1-T2V-1.3B~\cite{wan2025wan}, which natively supports 5s generation, it is trained with DMD on 5s autoregressive rollout clips with a training horizon of $L_{\text{train}}=21$ latent frames and a sliding window of $w=9$. Unless otherwise specified, we set $\alpha=0.1$, $\beta=2.5$, $\rho=-1$, and $L_{\text{sink}}=3$ as the default setting.

We compare against the base Self Forcing and CausVid~\cite{yin2025slow} models, representative training-based methods LongLive~\cite{yang2025longlive} and Rolling Forcing~\cite{liu2025rolling}, and the training-free method Deep Forcing~\cite{yi2025deepforcingtrainingfreelong}. All baselines in the main comparisons are few-step distilled 1.3B models. We further evaluate generalization to the larger Krea Realtime 14B (Krea-14B)~\cite{krea_realtime_14b} model and minute-level video generation with LongLive.

\paragraph{Evaluation metrics}
We primarily report VBench-Long~\cite{huang2024vbench} scores on the MovieGen and official VBench prompt sets across visual quality and semantic dimensions. In addition, we measure imaging-quality drift between the first and last 2s clips of each generated video:
\begin{equation}
\Delta_{\mathrm{Drift}}=\frac{1}{N}\sum_{i=1}^{N}
\left[\mathrm{IQ}(c_{i,\mathrm{last}})-\mathrm{IQ}(c_{i,\mathrm{first}})\right],
\label{eq:iq_drift}
\end{equation}
where $N$ is the number of evaluated videos and $\mathrm{IQ}$ denotes the clip-level Imaging Quality score on a 0--100 scale. Negative values indicate quality degradation from the start to the end of the sequence, whereas values near zero or above indicate sustained or improved final-clip quality.

\subsection{Quantitative comparison}

\paragraph{Evaluation on MovieGen prompts}
We evaluate 30s and 60s generation on MovieGen prompts~\cite{polyak2024movie} refined by Qwen2.5-7B-Instruct, corresponding to $6\times$ and $12\times$ extensions of Self Forcing's training horizon. As shown in Table~\ref{tab:main_comparison_moviegen}, Self Forcing + FLEX achieves the highest Quality Score of \textbf{83.01} at 30s, together with the smallest imaging-quality drift of $-0.06$. Compared with vanilla Self Forcing, FLEX lifts Dynamic Degree from 26.24 to 40.84 and Aesthetic Quality from 60.13 to \textbf{63.09} while achieving the best Temporal Flickering (\textbf{99.13}). Unlike Deep Forcing (82.53), whose elevated motion comes with lower consistency and pronounced drift ($-2.65$), FLEX improves motion dynamics while maintaining strong visual stability, outperforming both training-free and training-based baselines including LongLive (82.78) and Rolling Forcing (82.31) across the 30s benchmark.

At 60s, FLEX raises Self Forcing's Quality Score from 80.51 to \textbf{82.68} and reduces imaging-quality drift from $-2.51$ to $-0.56$. Without additional training, it matches LongLive, which is explicitly trained on sequences up to 60s, and surpasses both Rolling Forcing (82.10) and Deep Forcing (82.16). FLEX also delivers consistent gains on CausVid, improving long-range consistency while preserving motion and visual fidelity: Quality Score rises from 81.04 to 81.96 at 30s and from 80.73 to 82.01 at 60s, accompanied by substantial Dynamic Degree gains from 32.65 to 43.24 and from 22.11 to 40.44, respectively. At 60s, Subject and Background Consistency also improve, and $\Delta_{\mathrm{Drift}}$ increases from $-0.06$ to \textbf{+1.63}, indicating that the final 2s clip scores higher in Imaging Quality than the first 2s clip on average.

\paragraph{Evaluation on VBench prompts}
We further evaluate 30s generation on the official VBench evaluation prompt set. As shown in Table~\ref{tab:main_comparison_vbench}, FLEX achieves the highest Total Score of \textbf{83.48}, outperforming both training-based and training-free counterparts: LongLive (82.88), Rolling Forcing (82.83), and Deep Forcing (82.50). Specifically, FLEX boosts Dynamic Degree from Self Forcing's 29.38 to 44.32 while attaining the highest Aesthetic Quality (\textbf{65.43}) and Subject Consistency (\textbf{98.70}) among all compared methods.

FLEX also improves Semantic Score from 76.42 to 80.73 and leads in Object Class, Spatial Relationship, and Overall Consistency, indicating better preservation of prompt semantics over extended horizons. Full results across all 16 VBench-Long dimensions are provided in Appendix~\ref{app:vbench_long}.

\subsection{Qualitative comparison}
\label{sec:Qualitative Results}

Figure~\ref{fig:qualitative_30s} compares FLEX with base AR video models (CausVid and Self Forcing), training-based long-video methods (LongLive and Rolling Forcing), and the training-free baseline Deep Forcing on 30s generation. Across both scenes, existing methods exhibit varying degrees of color oversaturation, geometric distortion, blurring, and gradual drift in subject or background appearance as the rollout progresses. In contrast, FLEX preserves sharp structural details, such as the woman's attire by the lake and fine spacesuit textures in the futuristic alley, while maintaining natural color balance, temporal stability, and coherent scene geometry throughout the 30s generation horizon.

\begin{table}[!t]
\centering
\footnotesize
\setlength{\tabcolsep}{5pt}
\caption{\textbf{Component ablation.} NR, ANS, and AS denote frequency-aware RoPE, Antiphase Noise Sampling, and Attention Sink, respectively. Full FLEX achieves the best Quality Score.}
\label{tab:module_ablation}
\vspace{-0.3em}
\resizebox{\columnwidth}{!}{%
\begin{tabular}{ccc|cccccc|c}
\toprule
NR & ANS & AS & Subj. & Back. & Aest. & Imag. & Dyn. & Flick. & \textbf{Quality} \\
\midrule
\xmark   & \cmark & \cmark & 96.37 & 95.62 & 60.46 & 67.66 & 59.67 & 98.59 & 82.52 \\
\cmark   & \xmark & \cmark & 98.31 & 97.32 & 62.46 & 70.08 & 33.84 & 99.13 & 82.78 \\
\cmark   & \cmark & \xmark & 98.10 & 97.37 & 59.03 & 67.96 & 32.42 & 98.61 & 81.14 \\
\midrule
\cmark   & \cmark & \cmark & 98.21 & 97.23 & 62.58 & 69.77 & 40.63 & 99.12 & \textbf{83.07} \\
\bottomrule
\end{tabular}%
}
\vspace{-0.8em}
\end{table}

We further evaluate 60s generation under a complex multi-subject prompt in Figure~\ref{fig:qualitative_60s}. In this challenging $12\times$ extrapolation setting, CausVid and Self Forcing suffer from severe visual collapse, while LongLive and Rolling Forcing exhibit progressive identity drift as distinctive facial markings and breed traits fade over time. By comparison, FLEX sustains high visual fidelity and subject consistency while faithfully capturing prompt-specified actions, such as a dog tilting its head in confusion. The temporally ordered crops of characters A to E, tracking one subject from each method, further highlight FLEX's ability to preserve fine-grained appearance attributes over the entire minute. Additional 30s and 60s qualitative comparisons across diverse prompts and motion patterns are presented and discussed in Appendix~\ref{appendix:extended_qualitative}.

\begin{table}[!t]
\centering
\footnotesize
\setlength{\tabcolsep}{4.5pt}
\caption{\textbf{Sensitivity to RoPE exposure thresholds $\alpha$ and $\beta$ on VBench-Long}, varying one threshold while keeping the other fixed.}
\label{tab:ablation_rope}
\vspace{-0.3em}
\resizebox{\columnwidth}{!}{%
\begin{tabular}{cc|cccccc|c}
\toprule
$\alpha$ & $\beta$ & Subj. & Back. & Motion & Aest. & Imag. & Dyn. & \textbf{Quality} \\
\midrule
\multicolumn{9}{l}{\emph{Varying $\beta$ (fixed $\alpha = 0.1$)}} \\
\midrule
0.1 & 0.5 & 97.21 & 96.35 & 97.60 & 60.96 & 68.31 & 62.40 & 83.31 \\
0.1 & 1.0 & 97.36 & 96.47 & 97.81 & 61.35 & 68.56 & 58.40 & \textbf{83.32} \\
0.1 & 1.5 & 97.61 & 96.67 & 98.02 & 61.53 & 68.78 & 52.69 & 83.15 \\
0.1 & 2.0 & 97.94 & 96.96 & 98.22 & 62.17 & 69.00 & 47.36 & 83.14 \\
0.1 & 3.0 & 98.44 & 97.40 & 98.62 & 62.68 & 69.84 & 34.86 & 82.85 \\
\midrule
\multicolumn{9}{l}{\emph{Varying $\alpha$ (fixed $\beta = 2.5$)}} \\
\midrule
0.001 & 2.5 & 98.16 & 97.16 & 98.41 & 62.22 & 69.77 & 40.38 & 82.93 \\
0.005 & 2.5 & 98.16 & 97.15 & 98.41 & 62.41 & 69.73 & 39.40 & 82.90 \\
0.01  & 2.5 & 98.15 & 97.15 & 98.36 & 62.21 & 69.61 & 40.87 & 82.92 \\
0.1   & 2.5 & 98.21 & 97.23 & 98.44 & 62.58 & 69.77 & 40.63 & \textbf{83.07} \\
0.5   & 2.5 & 98.49 & 97.45 & 98.65 & 62.61 & 69.90 & 34.13 & 82.79 \\
\bottomrule
\end{tabular}%
}
\vspace{-0.7em}
\end{table}

\begin{figure}[!t]
    \centering
    \includegraphics[width=\columnwidth]{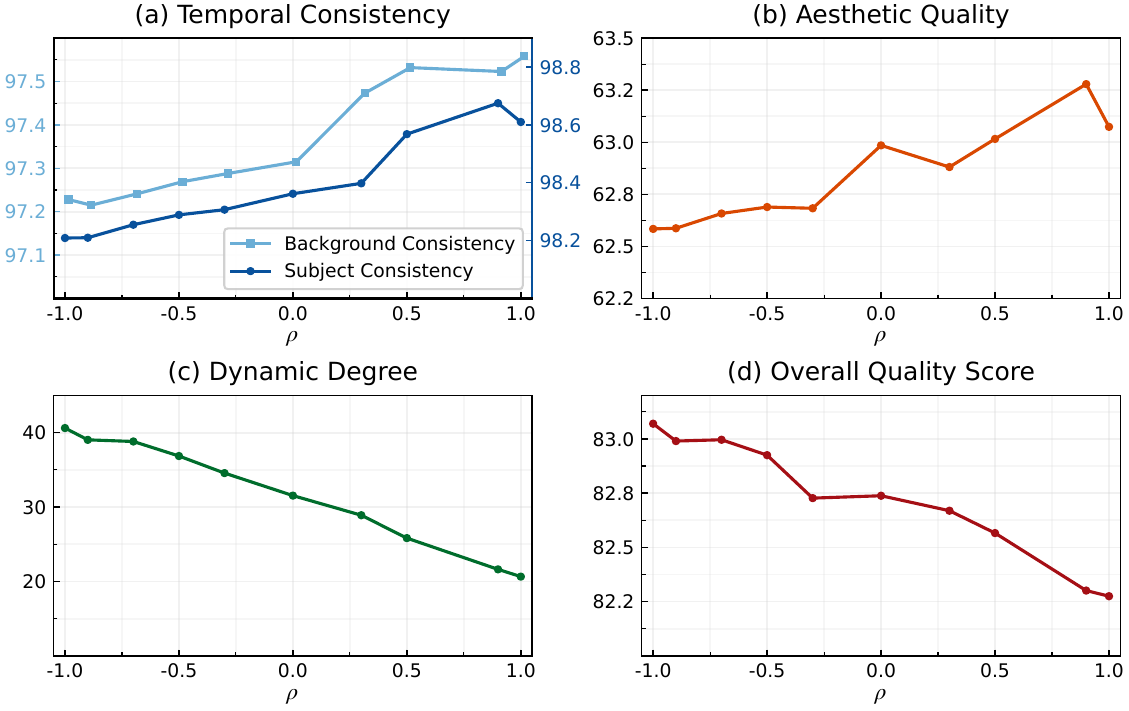}
    \vspace{-0.5em}
    \caption{\textbf{Effect of within-chunk noise correlation $\rho$.} Antiphase sampling ($\rho=-1$) improves Dynamic Degree and Quality Score relative to independent noise ($\rho=0$), while retaining high subject and background consistency.}
    \label{fig:ablation_noise}
    \vspace{-0.9em}
\end{figure}

\subsection{Ablation study}
\label{sec:ablation}

We conduct 30s ablations with Self Forcing on 128 MovieGen prompts to examine module contributions and hyperparameter sensitivity. In the following ablation tables, Subj., Back., Motion, Aest., Imag., Dyn., and Flick. denote Subject Consistency, Background Consistency, Motion Smoothness, Aesthetic Quality, Imaging Quality, Dynamic Degree, and Temporal Flickering, respectively.

\paragraph{Component ablation}
We evaluate the contributions of the three core modules: frequency-aware 3D RoPE modulation (denoted NR, following its NTK-by-parts formulation), Antiphase Noise Sampling (ANS), and Attention Sink (AS). As shown in Table~\ref{tab:module_ablation}, the full configuration NR + ANS + AS achieves the highest Quality Score of \textbf{83.07}. \textbf{(i) NR}: removing frequency-aware RoPE modulation lowers both Subject and Background Consistency (from 98.21/97.23 to 96.37/95.62), supporting its role in resolving relative temporal positions over extended horizons and curbing temporal drift. \textbf{(ii) ANS}: disabling antiphase correlation markedly reduces Dynamic Degree from 40.63 to 33.84, consistent with our spectral analysis that shifting noise power toward high frequencies counteracts motion decay. \textbf{(iii) AS}: removing the attention sink causes a clear drop in Aesthetic Quality and aggregate Quality Score (to 81.14), underscoring the importance of initial frames as persistent global anchors against accumulated visual drift. Together, these results confirm the complementary roles of all three modules in FLEX.

\paragraph{Impact of RoPE hyperparameters $\alpha$ and $\beta$}
Table~\ref{tab:ablation_rope} evaluates sensitivity to the exposure thresholds $\alpha$ and $\beta$, revealing a clear consistency-versus-motion trade-off: increasing either threshold expands the interpolated regime and generally favors higher Subject and Background Consistency at the expense of Dynamic Degree. An intermediate setting such as $\alpha=0.1$ and $\beta=2.5$ strikes an effective balance between long-range consistency and motion dynamics. Empirically, this same $(\alpha, \beta)$ configuration improves both 30s and 60s generation on Self Forcing (see Tables~\ref{tab:main_comparison_moviegen} and~\ref{tab:main_comparison_vbench}).

\paragraph{Effect of noise correlation $\rho$}
Figure~\ref{fig:ablation_noise} examines how within-chunk noise correlation $\rho$ affects generation behavior. As $\rho$ increases from $-1$ to $1$, Subject and Background Consistency generally increase, whereas Dynamic Degree drops sharply from 40.63 to 20.65. Positive correlation encourages similar initial noise states across adjacent frames, whereas negative correlation introduces stronger frame-to-frame variation at initialization. Relative to standard i.i.d.\ noise ($\rho=0$), antiphase sampling ($\rho=-1$) boosts Dynamic Degree from 31.54 to 40.63 and raises Quality Score from 82.74 to 83.07, demonstrating that ANS effectively mitigates motion attenuation while maintaining strong overall quality during autoregressive generation beyond the training horizon.

\subsection{Generalization to Krea-14B}
\label{sec:krea_generalization}

\begin{table}[!t]
\centering
\caption{\textbf{Generalization to Krea-14B.} VBench-Long evaluation on the 30s generation task. Higher scores are better.}
\label{tab:krea_comparison}
\vspace{-0.5em}
\small
\setlength{\tabcolsep}{2pt}
\resizebox{\columnwidth}{!}{%
\begin{tabular}{l|ccccccc|c}
\toprule
Method & Subj. & Back. & Motion & Aest. & Imag. & Dyn. & Flick. & \textbf{Quality} \\
\midrule
Krea-14B & \textbf{96.46} & \textbf{95.59} & \textbf{98.12} & 58.05 & \textbf{69.22} & 72.39 & \textbf{98.39} & 83.60 \\
Krea-14B + FLEX & 95.73 & 95.32 & 97.75 & \textbf{59.35} & 67.86 & \textbf{83.86} & 98.12 & \textbf{83.98} \\
\bottomrule
\end{tabular}%
}
\vspace{-1.0em}
\end{table}

To evaluate generalization beyond the 1.3B baselines, we integrate FLEX into Krea-14B~\cite{krea_realtime_14b}, a model distilled from Wan2.1-14B, and evaluate 30s generation on 128 prompts. Since Krea-14B already applies local temporal RoPE reindexing, we retain its native positional scheme and incorporate ANS and Attention Sink into its inference pipeline, as further detailed in Appendix~\ref{sec:appendix_krea}.

As shown in Table~\ref{tab:krea_comparison}, Krea-14B + FLEX increases Dynamic Degree from 72.39 to \textbf{83.86} and Aesthetic Quality from 58.05 to \textbf{59.35}, raising the aggregate Quality Score from 83.60 to \textbf{83.98}. Stronger dynamics are accompanied by slight decreases in Subject and Background Consistency and Motion Smoothness, reflecting the same trade-off between temporal stability and motion observed in our ablations. These gains confirm that FLEX transfers effectively to a 14B backbone with a distinct inference pipeline, even when the base model already generates substantial motion.

\subsection{Generalization to minute-level generation}
\label{sec:minute_level}

\begin{table}[!t]
\centering
\caption{\textbf{Minute-level generation with LongLive.} VBench-Long evaluation at a 240s horizon ($4\times$ extension beyond 60s training).}
\label{tab:longlive_comparison}
\vspace{-0.5em}
\setlength{\tabcolsep}{2pt} 
\small
\resizebox{\columnwidth}{!}{
\begin{tabular}{l|ccccc|c|c}
\toprule
Method & Subj. & Back. & Aest. & Dyn. & Flick. & \textbf{Quality} & $\Delta_{\mathrm{Drift}}$ \\

\midrule
LongLive & 97.42 & 96.50 & \textbf{61.17} & 39.03  & 99.02 & 82.40 & $-2.19$ \\

LongLive + FLEX & \textbf{97.57} & \textbf{96.64} & 61.12 & \textbf{40.90}  & \textbf{99.09} & \textbf{82.48} & $\mathbf{-1.17}$ \\
\bottomrule
\end{tabular}
}
\vspace{-1.0em}
\end{table}

\begin{figure}[!t]
    \centering
    \includegraphics[width=0.95\columnwidth]{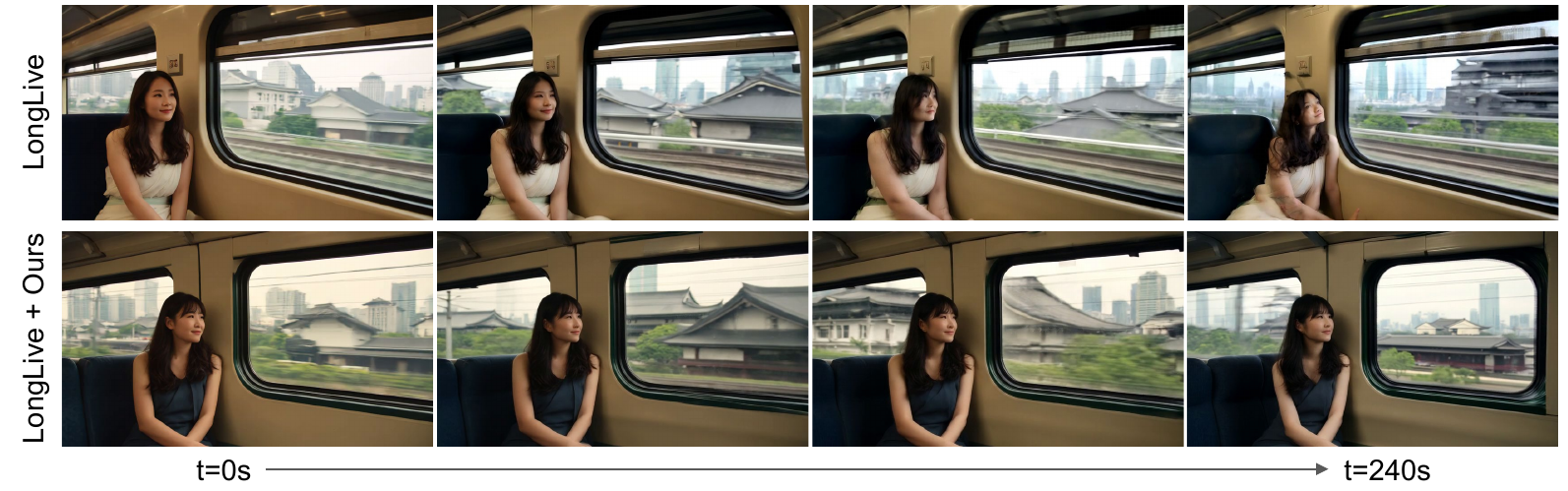}
    \vspace{-0.6em}
    \caption{\textbf{Qualitative comparison at 240s.} Built on LongLive, FLEX better preserves subject appearance and scene details.}
    \label{fig:longlive_240s}
    \vspace{-1.2em}
\end{figure}

We further evaluate minute-level generation by integrating FLEX into LongLive at a 240s horizon ($4\times$ its 60s training length), retaining its native attention sink while applying NR and ANS, as described in detail in Appendix~\ref{sec:appendix_longlive}.

As shown in Table~\ref{tab:longlive_comparison}, FLEX improves generation performance on most reported metrics, lifting Dynamic Degree from 39.03 to 40.90, raising Quality Score from 82.40 to 82.48, and increasing $\Delta_{\mathrm{Drift}}$ from $-2.19$ to $-1.17$ for stronger quality retention. Figure~\ref{fig:longlive_240s} provides a visual comparison: LongLive exhibits blurred facial features and clothing textures in late frames, whereas LongLive + FLEX preserves sharp character appearance as the scenery outside the train window evolves. Together, these results show that FLEX complements native long-video training and extends the inference-time scalability of autoregressive video models.

\section{Conclusion and limitations}
We presented \textbf{FLEX}, a training-free inference framework that effectively extends the generation horizon of pretrained autoregressive video diffusion models beyond short training clips.
Building on the identification of spectral bias in 3D RoPE and insufficient dynamic priors in noise sampling, FLEX combines Frequency-aware 3D RoPE Modulation, Antiphase Noise Sampling, and an inference-only Attention Sink to jointly maintain consistency and motion dynamics, bridging short-clip training and long-horizon inference.
Extensive VBench-Long evaluations show that FLEX matches or surpasses training-based baselines such as LongLive and Rolling Forcing at $6\times$ to $12\times$ extrapolation. As a plug-and-play method, FLEX integrates seamlessly into existing autoregressive pipelines, generalizing to CausVid and Krea-14B and improving minute-level generation with LongLive.

\paragraph{Limitations}
Because FLEX is training-free and operates strictly at inference time, its visual quality remains upper-bounded by the underlying base model.
In addition, its hyperparameters $(\alpha,\beta,\rho)$ reflect an inherent trade-off between temporal consistency and motion dynamics. Finally, our analysis is grounded in chunk-wise causal AR video diffusion, and extending these frequency-aware design principles to non-causal or hybrid video generation architectures remains an open direction for future work.

\section*{LLM Usage Statement}
Large language models (GPT and Gemini) were used to assist with language polishing and grammar refinement across the manuscript and supplementary material, as well as cross-checking reported numbers against experimental records. All content, mathematical derivations, and quantitative analyses were verified by the authors, who take full responsibility for the final submission.

{\small
\bibliographystyle{ieeenat_fullname}
\bibliography{references}

\begin{thebibliography}{38}
\providecommand{\natexlab}[1]{#1}
\providecommand{\url}[1]{\texttt{#1}}
\expandafter\ifx\csname urlstyle\endcsname\relax
  \providecommand{\doi}[1]{doi: #1}\else
  \providecommand{\doi}{doi: \begingroup \urlstyle{rm}\Url}\fi

\bibitem[Blattmann et~al.(2023)Blattmann, Dockhorn, Kulal, Mendelevitch, Kilian, Lorenz, Levi, English, Voleti, Letts, Jampani, and Rombach]{blattmann2023stablevideodiffusionscaling}
Andreas Blattmann, Tim Dockhorn, Sumith Kulal, Daniel Mendelevitch, Maciej Kilian, Dominik Lorenz, Yam Levi, Zion English, Vikram Voleti, Adam Letts, Varun Jampani, and Robin Rombach.
\newblock {Stable Video Diffusion}: Scaling latent video diffusion models to large datasets.
\newblock \emph{arXiv preprint arXiv:2311.15127}, 2023.

\bibitem[bloc97(2023{\natexlab{a}})]{bloc97ntk}
bloc97.
\newblock {NTK}-aware scaled {RoPE} allows {LLaMA} models to have extended (8k+) context size without any fine-tuning and minimal perplexity degradation.
\newblock \href{https://www.reddit.com/r/LocalLLaMA/comments/14lz7j5/ntkaware_scaled_rope_allows_llama_models_to_have/}{Reddit post, r/LocalLLaMA}, 2023{\natexlab{a}}.

\bibitem[bloc97(2023{\natexlab{b}})]{bloc97parts}
bloc97.
\newblock Add {NTK}-aware interpolation "by parts" correction.
\newblock \href{https://github.com/jquesnelle/yarn/pull/1}{GitHub pull request \#1, jquesnelle/yarn}, 2023{\natexlab{b}}.

\bibitem[Burgert et~al.(2025)Burgert, Xu, Xian, Pilarski, Clausen, He, Ma, Deng, Li, Mousavi, Ryoo, Debevec, and Yu]{burgert2025gowiththeflow}
Ryan Burgert, Yuancheng Xu, Wenqi Xian, Oliver Pilarski, Pascal Clausen, Mingming He, Li Ma, Yitong Deng, Lingxiao Li, Mohsen Mousavi, Michael Ryoo, Paul Debevec, and Ning Yu.
\newblock {Go-with-the-Flow}: Motion-controllable video diffusion models using real-time warped noise.
\newblock In \emph{Proceedings of the IEEE/CVF Conference on Computer Vision and Pattern Recognition}, pages 13--23, 2025.

\bibitem[Chang et~al.(2024)Chang, Tang, Gross, and Azevedo]{chang2024warpednoise}
Pascal Chang, Jingwei Tang, Markus Gross, and Vinicius~C. Azevedo.
\newblock How {I} warped your noise: A temporally-correlated noise prior for diffusion models.
\newblock In \emph{International Conference on Learning Representations}, 2024.

\bibitem[Chen et~al.(2025{\natexlab{a}})Chen, Lin, Yang, Lin, Zhu, Fan, Zhang, Chen, Chen, Ma, Xiong, Wang, Pang, Kang, Xu, Jin, Liang, Song, Zhao, Xu, Qiu, Li, Fei, Li, and Zhou]{chen2025skyreelsv2infinitelengthfilmgenerative}
Guibin Chen, Dixuan Lin, Jiangping Yang, Chunze Lin, Junchen Zhu, Mingyuan Fan, Hao Zhang, Sheng Chen, Zheng Chen, Chengcheng Ma, Weiming Xiong, Wei Wang, Nuo Pang, Kang Kang, Zhiheng Xu, Yuzhe Jin, Yupeng Liang, Yubing Song, Peng Zhao, Boyuan Xu, Di Qiu, Debang Li, Zhengcong Fei, Yang Li, and Yahui Zhou.
\newblock {SkyReels-V2}: Infinite-length film generative model.
\newblock \emph{arXiv preprint arXiv:2504.13074}, 2025{\natexlab{a}}.

\bibitem[Chen et~al.(2025{\natexlab{b}})Chen, Zhao, Yu, Chu, Chen, Yang, Wang, Pan, Zhou, Ling, Liu, Yi, Zhang, Li, Chen, Cai, Fidler, Luo, Han, and Xie]{chen2025sanavideoefficientvideogeneration}
Junsong Chen, Yuyang Zhao, Jincheng Yu, Ruihang Chu, Junyu Chen, Shuai Yang, Xianbang Wang, Yicheng Pan, Daquan Zhou, Huan Ling, Haozhe Liu, Hongwei Yi, Hao Zhang, Muyang Li, Yukang Chen, Han Cai, Sanja Fidler, Ping Luo, Song Han, and Enze Xie.
\newblock {SANA-Video}: Efficient video generation with block linear diffusion transformer.
\newblock \emph{arXiv preprint arXiv:2509.24695}, 2025{\natexlab{b}}.

\bibitem[Chen et~al.(2023{\natexlab{a}})Chen, Wong, Chen, and Tian]{chen2023extending}
Shouyuan Chen, Sherman Wong, Liangjian Chen, and Yuandong Tian.
\newblock Extending context window of large language models via positional interpolation.
\newblock \emph{arXiv preprint arXiv:2306.15595}, 2023{\natexlab{a}}.

\bibitem[Chen et~al.(2023{\natexlab{b}})Chen, Ji, Wu, Wu, Xie, Li, Xia, Xiao, and Lin]{chen2023control}
Weifeng Chen, Yatai Ji, Jie Wu, Hefeng Wu, Pan Xie, Jiashi Li, Xin Xia, Xuefeng Xiao, and Liang Lin.
\newblock {Control-A-Video}: Controllable text-to-video generation with diffusion models.
\newblock \emph{arXiv preprint arXiv:2305.13840}, 2023{\natexlab{b}}.

\bibitem[Ding et~al.(2024)Ding, Zhang, Zhang, Xu, Shang, Xu, Yang, and Yang]{ding2024longropeextendingllmcontext}
Yiran Ding, Li~Lyna Zhang, Chengruidong Zhang, Yuanyuan Xu, Ning Shang, Jiahang Xu, Fan Yang, and Mao Yang.
\newblock {LongRoPE}: Extending {LLM} context window beyond 2 million tokens.
\newblock In \emph{International Conference on Machine Learning}, 2024.

\bibitem[Gao et~al.(2025)Gao, Chen, Liu, Feng, Si, Fu, Qiao, and Liu]{gao2025longvie}
Jianxiong Gao, Zhaoxi Chen, Xian Liu, Jianfeng Feng, Chenyang Si, Yanwei Fu, Yu Qiao, and Ziwei Liu.
\newblock {LongVie}: Multimodal-guided controllable ultra-long video generation.
\newblock \emph{arXiv preprint arXiv:2508.03694}, 2025.

\bibitem[Guo et~al.(2024)Guo, Yang, Rao, Liang, Wang, Qiao, Agrawala, Lin, and Dai]{guo2023animatediff}
Yuwei Guo, Ceyuan Yang, Anyi Rao, Zhengyang Liang, Yaohui Wang, Yu Qiao, Maneesh Agrawala, Dahua Lin, and Bo Dai.
\newblock {AnimateDiff}: Animate your personalized text-to-image diffusion models without specific tuning.
\newblock In \emph{International Conference on Learning Representations}, 2024.

\bibitem[Ho et~al.(2022)Ho, Salimans, Gritsenko, Chan, Norouzi, and Fleet]{ho2022videodiffusionmodels}
Jonathan Ho, Tim Salimans, Alexey Gritsenko, William Chan, Mohammad Norouzi, and David~J. Fleet.
\newblock Video diffusion models.
\newblock \emph{Advances in Neural Information Processing Systems (NeurIPS)}, 35:\penalty0 8633--8646, 2022.

\bibitem[Huang et~al.(2025)Huang, Li, He, Zhou, and Shechtman]{huang2025self}
Xun Huang, Zhengqi Li, Guande He, Mingyuan Zhou, and Eli Shechtman.
\newblock {Self Forcing}: Bridging the train-test gap in autoregressive video diffusion.
\newblock \emph{Advances in Neural Information Processing Systems (NeurIPS)}, 2025.

\bibitem[Huang et~al.(2024)Huang, He, Yu, Zhang, Si, Jiang, Zhang, Wu, Jin, Chanpaisit, et~al.]{huang2024vbench}
Ziqi Huang, Yinan He, Jiashuo Yu, Fan Zhang, Chenyang Si, Yuming Jiang, Yuanhan Zhang, Tianxing Wu, Qingyang Jin, Nattapol Chanpaisit, et~al.
\newblock {VBench}: Comprehensive benchmark suite for video generative models.
\newblock In \emph{Proceedings of the IEEE/CVF Conference on Computer Vision and Pattern Recognition}, pages 21807--21818, 2024.

\bibitem[Kodaira et~al.(2025)Kodaira, Hou, Hou, Georgopoulos, Juefei-Xu, Tomizuka, and Zhao]{kodaira2025streamditrealtimestreamingtexttovideo}
Akio Kodaira, Tingbo Hou, Ji Hou, Markos Georgopoulos, Felix Juefei-Xu, Masayoshi Tomizuka, and Yue Zhao.
\newblock {StreamDiT}: Real-time streaming text-to-video generation.
\newblock \emph{arXiv preprint arXiv:2507.03745}, 2025.

\bibitem[Liu et~al.(2025)Liu, Hu, Xu, Shan, and Lu]{liu2025rolling}
Kunhao Liu, Wenbo Hu, Jiale Xu, Ying Shan, and Shijian Lu.
\newblock {Rolling Forcing}: Autoregressive long video diffusion in real time.
\newblock \emph{arXiv preprint arXiv:2509.25161}, 2025.

\bibitem[Lu and Yang(2025)]{lu2025freelong++}
Yu Lu and Yi Yang.
\newblock {FreeLong++}: Training-free long video generation via multi-band {SpectralFusion}.
\newblock \emph{arXiv preprint arXiv:2507.00162}, 2025.

\bibitem[Millon(2025)]{krea_realtime_14b}
Erwann Millon.
\newblock {Krea Realtime 14B}: Real-time video generation.
\newblock GitHub repository: krea-ai/realtime-video, 2025.

\bibitem[Peebles and Xie(2023)]{peebles2023scalablediffusionmodelstransformers}
William Peebles and Saining Xie.
\newblock Scalable diffusion models with transformers.
\newblock In \emph{Proceedings of the IEEE/CVF International Conference on Computer Vision}, pages 4195--4205, 2023.

\bibitem[Peng et~al.(2024)Peng, Quesnelle, Fan, and Shippole]{peng2023yarn}
Bowen Peng, Jeffrey Quesnelle, Honglu Fan, and Enrico Shippole.
\newblock {YaRN}: Efficient context window extension of large language models.
\newblock In \emph{International Conference on Learning Representations}, 2024.

\bibitem[Polyak et~al.(2024)Polyak, Zohar, Brown, Tjandra, Sinha, Lee, Vyas, Shi, Ma, Chuang, et~al.]{polyak2024movie}
Adam Polyak, Amit Zohar, Andrew Brown, Andros Tjandra, Animesh Sinha, Ann Lee, Apoorv Vyas, Bowen Shi, Chih-Yao Ma, Ching-Yao Chuang, et~al.
\newblock {Movie Gen}: A cast of media foundation models.
\newblock \emph{arXiv preprint arXiv:2410.13720}, 2024.

\bibitem[Qiu et~al.(2024)Qiu, Xia, Zhang, He, Wang, Shan, and Liu]{qiu2023freenoise}
Haonan Qiu, Menghan Xia, Yong Zhang, Yingqing He, Xintao Wang, Ying Shan, and Ziwei Liu.
\newblock {FreeNoise}: Tuning-free longer video diffusion via noise rescheduling.
\newblock In \emph{International Conference on Learning Representations}, 2024.

\bibitem[{Sand.ai} et~al.(2025){Sand.ai}, Teng, Jia, Sun, Li, Li, Tang, Han, Zhang, Zhang, Luo, Kang, Sun, Cao, Huang, Lin, Fang, Tao, Zhang, Wang, Liu, Shi, Su, Sun, Pan, Wang, Sheng, Cui, Hu, Yan, Yin, Zhang, Liu, Yin, Yang, Song, Hu, Zhang, and Li]{ai2025magi1autoregressivevideogeneration}
{Sand.ai}, Hansi Teng, Hongyu Jia, Lei Sun, Lingzhi Li, Maolin Li, Mingqiu Tang, Shuai Han, Tianning Zhang, W.~Q. Zhang, Weifeng Luo, Xiaoyang Kang, Yuchen Sun, Yue Cao, Yunpeng Huang, Yutong Lin, Yuxin Fang, Zewei Tao, Zheng Zhang, Zhongshu Wang, Zixun Liu, Dai Shi, Guoli Su, Hanwen Sun, Hong Pan, Jie Wang, Jiexin Sheng, Min Cui, Min Hu, Ming Yan, Shucheng Yin, Siran Zhang, Tingting Liu, Xianping Yin, Xiaoyu Yang, Xin Song, Xuan Hu, Yankai Zhang, and Yuqiao Li.
\newblock {MAGI-1}: Autoregressive video generation at scale.
\newblock \emph{arXiv preprint arXiv:2505.13211}, 2025.

\bibitem[Song et~al.(2025)Song, Chen, Simchowitz, Du, Tedrake, and Sitzmann]{song2025history}
Kiwhan Song, Boyuan Chen, Max Simchowitz, Yilun Du, Russ Tedrake, and Vincent Sitzmann.
\newblock History-guided video diffusion.
\newblock In \emph{International Conference on Machine Learning}, 2025.

\bibitem[Su et~al.(2024)Su, Ahmed, Lu, Pan, Bo, and Liu]{su2024roformer}
Jianlin Su, Murtadha Ahmed, Yu Lu, Shengfeng Pan, Wen Bo, and Yunfeng Liu.
\newblock {RoFormer}: Enhanced transformer with rotary position embedding.
\newblock \emph{Neurocomputing}, 568:\penalty0 127063, 2024.

\bibitem[Sun et~al.(2025)Sun, Wang, Li, Liu, Sun, Feng, Lao, Zhou, He, and Liu]{sun2025ardiffusionasynchronousvideogeneration}
Mingzhen Sun, Weining Wang, Gen Li, Jiawei Liu, Jiahui Sun, Wanquan Feng, Shanshan Lao, SiYu Zhou, Qian He, and Jing Liu.
\newblock {AR-Diffusion}: Asynchronous video generation with auto-regressive diffusion.
\newblock In \emph{Proceedings of the IEEE/CVF Conference on Computer Vision and Pattern Recognition}, 2025.

\bibitem[Touvron et~al.(2023)Touvron, Lavril, Izacard, Martinet, Lachaux, Lacroix, Rozi{\`e}re, Goyal, Hambro, Azhar, et~al.]{touvron2023llama}
Hugo Touvron, Thibaut Lavril, Gautier Izacard, Xavier Martinet, Marie-Anne Lachaux, Timoth{\'e}e Lacroix, Baptiste Rozi{\`e}re, Naman Goyal, Eric Hambro, Faisal Azhar, et~al.
\newblock {LLaMA}: Open and efficient foundation language models.
\newblock \emph{arXiv preprint arXiv:2302.13971}, 2023.

\bibitem[{Wan Team} et~al.(2025){Wan Team}, Wang, Ai, Wen, Mao, Xie, Chen, Yu, Zhao, Yang, et~al.]{wan2025wan}
{Wan Team}, Ang Wang, Baole Ai, Bin Wen, Chaojie Mao, Chen-Wei Xie, Di Chen, Feiwu Yu, Haiming Zhao, Jianxiao Yang, et~al.
\newblock {Wan}: Open and advanced large-scale video generative models.
\newblock \emph{arXiv preprint arXiv:2503.20314}, 2025.

\bibitem[Xiao et~al.(2024)Xiao, Tian, Chen, Han, and Lewis]{xiao2023efficient}
Guangxuan Xiao, Yuandong Tian, Beidi Chen, Song Han, and Mike Lewis.
\newblock Efficient streaming language models with attention sinks.
\newblock In \emph{International Conference on Learning Representations}, 2024.

\bibitem[Yang et~al.(2025{\natexlab{a}})Yang, Li, Yang, Zhang, Hui, Zheng, Yu, Gao, Huang, Lv, et~al.]{yang2025qwen3}
An Yang, Anfeng Li, Baosong Yang, Beichen Zhang, Binyuan Hui, Bo Zheng, Bowen Yu, Chang Gao, Chengen Huang, Chenxu Lv, et~al.
\newblock {Qwen3} technical report.
\newblock \emph{arXiv preprint arXiv:2505.09388}, 2025{\natexlab{a}}.

\bibitem[Yang et~al.(2025{\natexlab{b}})Yang, Huang, Chu, Xiao, Zhao, Wang, Li, Xie, Chen, Lu, et~al.]{yang2025longlive}
Shuai Yang, Wei Huang, Ruihang Chu, Yicheng Xiao, Yuyang Zhao, Xianbang Wang, Muyang Li, Enze Xie, Yingcong Chen, Yao Lu, et~al.
\newblock {LongLive}: Real-time interactive long video generation.
\newblock \emph{arXiv preprint arXiv:2509.22622}, 2025{\natexlab{b}}.

\bibitem[Yesiltepe et~al.(2025)Yesiltepe, Meral, Akan, Oktay, and Yanardag]{yesiltepe2025infinityropeactioncontrollableinfinitevideo}
Hidir Yesiltepe, Tuna Han~Salih Meral, Adil~Kaan Akan, Kaan Oktay, and Pinar Yanardag.
\newblock {Infinity-RoPE}: Action-controllable infinite video generation emerges from autoregressive self-rollout.
\newblock \emph{arXiv preprint arXiv:2511.20649}, 2025.

\bibitem[Yi et~al.(2025)Yi, Jang, Cho, Nam, Yoon, and Kim]{yi2025deepforcingtrainingfreelong}
Jung Yi, Wooseok Jang, Paul~Hyunbin Cho, Jisu Nam, Heeji Yoon, and Seungryong Kim.
\newblock {Deep Forcing}: Training-free long video generation with deep sink and participative compression.
\newblock \emph{arXiv preprint arXiv:2512.05081}, 2025.

\bibitem[Yin et~al.(2024{\natexlab{a}})Yin, Gharbi, Park, Zhang, Shechtman, Durand, and Freeman]{yin2024improved}
Tianwei Yin, Micha{\"e}l Gharbi, Taesung Park, Richard Zhang, Eli Shechtman, Fredo Durand, and William~T. Freeman.
\newblock Improved distribution matching distillation for fast image synthesis.
\newblock \emph{Advances in Neural Information Processing Systems (NeurIPS)}, 37:\penalty0 47455--47487, 2024{\natexlab{a}}.

\bibitem[Yin et~al.(2024{\natexlab{b}})Yin, Gharbi, Zhang, Shechtman, Durand, Freeman, and Park]{yin2024one}
Tianwei Yin, Micha{\"e}l Gharbi, Richard Zhang, Eli Shechtman, Fredo Durand, William~T. Freeman, and Taesung Park.
\newblock One-step diffusion with distribution matching distillation.
\newblock In \emph{Proceedings of the IEEE/CVF Conference on Computer Vision and Pattern Recognition}, pages 6613--6623, 2024{\natexlab{b}}.

\bibitem[Yin et~al.(2025)Yin, Zhang, Zhang, Freeman, Durand, Shechtman, and Huang]{yin2025slow}
Tianwei Yin, Qiang Zhang, Richard Zhang, William~T. Freeman, Fredo Durand, Eli Shechtman, and Xun Huang.
\newblock From slow bidirectional to fast autoregressive video diffusion models.
\newblock In \emph{Proceedings of the IEEE/CVF Conference on Computer Vision and Pattern Recognition}, pages 22963--22974, 2025.

\bibitem[Zhong et~al.(2024)Zhong, Zhang, Lei, Liu, Gao, Hu, Chen, and Zhang]{zhong2024understandingropeextensionslongcontext}
Meizhi Zhong, Chen Zhang, Yikun Lei, Xikai Liu, Yan Gao, Yao Hu, Kehai Chen, and Min Zhang.
\newblock Understanding the {RoPE} extensions of long-context {LLMs}: An attention perspective.
\newblock \emph{arXiv preprint arXiv:2406.13282}, 2024.

\end{thebibliography}
}
\clearpage
\setcounter{page}{1}
\renewcommand{\thepage}{S\arabic{page}}
\maketitlesupplementary
\appendix
\section{Preliminary Formulations}
\label{sec:appendix_prelim}

We provide full formulations of the components summarized in Section~\ref{sec:preliminary}.

\paragraph{Intra-chunk Independent Noise Sampling}
The generation of each chunk $\mathcal{C}_i$ originates from initial noise latents at diffusion timestep $T$, $\mathbf{Z}_T^{(i)}=\{\mathbf{z}_T^{(i,0)},\dots,\mathbf{z}_T^{(i,f-1)}\}$. In standard inference protocols, the $f$ intra-chunk latents are sampled i.i.d.\ from a standard isotropic Gaussian:
\begin{equation}
\mathbb{E}\left[\mathbf{z}_T^{(i,u)}(\mathbf{z}_T^{(i,v)})^\top\right]=\mathbf{0},\quad \forall u\neq v\in\{0,\dots,f-1\},
\end{equation}
with $\mathbf{z}_T^{(i,u)}\sim\mathcal{N}(\mathbf{0},\mathbf{I})$. This intra-chunk independence serves as the stochastic prior for subsequent iterative denoising.

\paragraph{Causal Attention with Local Window}
The denoising step for the $i$-th chunk is
\begin{equation}
    \mathbf{Z}_{t-1}^{(i)} = \mathcal{G}_\theta\!\left(\mathbf{Z}_t^{(i)}\mid \mathbf{\Phi}_{<i},\mathbf{c},t\right),
\end{equation}
where $\mathcal{G}_\theta$ is the denoising network, $\mathbf{c}$ represents external conditions (e.g., text prompts), and $\mathbf{\Phi}_{<i}$ is the historical context. Due to computational constraints, current architectures employ sliding-window causal attention of size $w$. Indexing denoised latent frames globally by $n\in\{0,\dots,L-1\}$ (so chunk $\mathcal{C}_i$ spans $n\in[(i-1)f,\,if-1]$), vanilla sliding-window inference truncates the context to the most recent $w-f$ frames, $\mathbf{\Phi}_{<i}=\operatorname{KV}(\{\mathbf{z}_0^{[n]}\})$ for $n\in[\max(0,\,(i-1)f-(w-f)),\,(i-1)f-1]$; when an attention sink of $L_{\text{sink}}$ initial latent frames is enabled (Section~\ref{sec:sink}), $\mathbf{\Phi}_{<i}$ retains the initial frames $n\in[0,\,L_{\text{sink}}-1]$ in addition to the most recent $w-f$ frames. The total attention window therefore contains up to $w+L_{\text{sink}}$ frames, including the current chunk.

\paragraph{3D Rotary Positional Embedding}
Video diffusion Transformers such as Wan2.1~\cite{wan2025wan} decompose each attention head dimension $d_{\text{head}}$ into temporal, height, and width subspaces, $d_{\text{head}} = d_F + d_H + d_W$. Along the temporal subspace $d_F$ ($d_F/2$ orthogonal rotary planes), the rotary operator $f(\mathbf{q}, n, \theta) = \mathbf{R}_{\Theta, n}\mathbf{q}$ at global latent frame index $n$ transforms the $m$-th coordinate pair ($m \in \{0, \dots, \frac{d_F}{2}-1\}$) as
\begin{equation}
    \begin{pmatrix} q'_{2m} \\ q'_{2m+1} \end{pmatrix} =
    \begin{pmatrix} \cos(n \theta_m) & -\sin(n \theta_m) \\ \sin(n \theta_m) & \cos(n \theta_m) \end{pmatrix}
    \begin{pmatrix} q_{2m} \\ q_{2m+1} \end{pmatrix},
\end{equation}
where $\theta_m = b^{-2m/d_F}$ with base $b = 10000$. Consequently, the temporal contribution to the inner product depends only on the relative frame distance $\Delta n = n_q - n_k$:
\begin{equation}
    \langle f(\mathbf{q}, n_q, \theta), f(\mathbf{k}, n_k, \theta) \rangle = \mathbf{q}^\top \mathbf{R}_{\Theta, -\Delta n} \mathbf{k}.
\end{equation}

\section{Qualitative Analysis of PI and YaRN for Video Extension}
\label{sec:appendix_pi}

We examine the qualitative behavior of Position Interpolation (PI) and YaRN when applied to long-horizon AR video generation. Figure~\ref{fig:pi_failure_analysis} compares both methods with vanilla Self Forcing on two examples. PI exhibits motion collapse and severe chromatic artifacts, while YaRN exhibits increasingly oversaturated colors, texture drift, and structural degradation. We discuss these observations below.

\paragraph{Comparison Settings} We generate 30s videos with Self Forcing at $832\times480$ resolution, without attention sinks or ANS. PI uses $S=L/L_{\text{train}}=126/21=6$, determined by the target and training horizons. YaRN uses $\alpha=0.01$, $\beta=1.00$, and $\texttt{logit\_scale}=1.1$.

\begin{figure*}[t]
    \centering
    \includegraphics[width=0.8\linewidth]{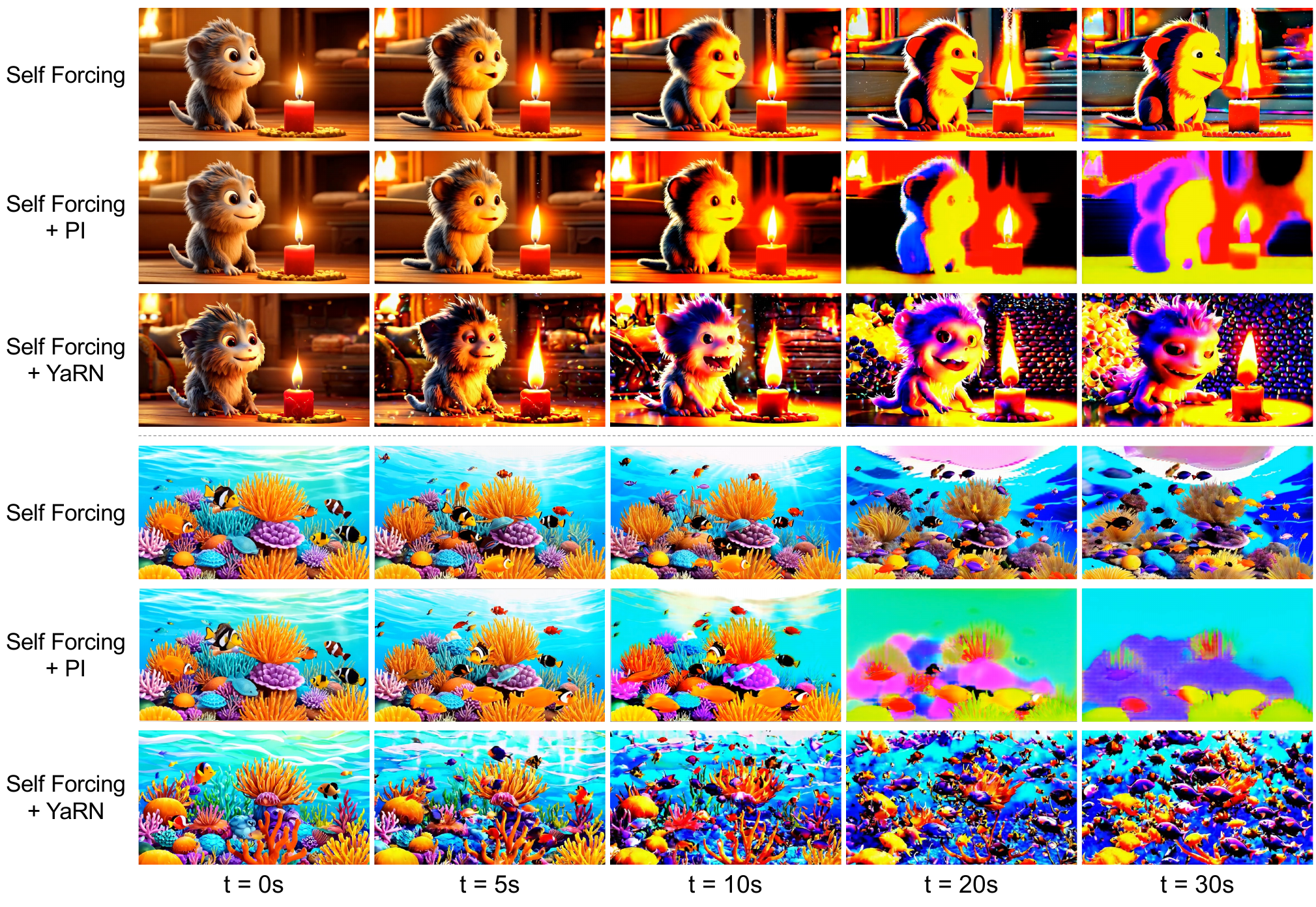} 
    \caption{\textbf{Qualitative comparison of Self Forcing, Position Interpolation (PI), and YaRN in long-horizon extrapolation.} 
    PI exhibits severe color saturation and structural degradation over time. YaRN also develops increasingly oversaturated colors, spurious textures, and appearance drift during extended rollouts.}
    \label{fig:pi_failure_analysis}
\end{figure*}

\begin{itemize}
    \item \textbf{Static Collapse and Temporal Compression:} The core mechanism of PI involves scaling the position indices $n' = n / S$. During long-range inference, this compresses temporal coordinates and reduces inter-frame phase differences. A possible explanation for the observed static collapse is that this uniform compression, especially of high-frequency components, weakens \textbf{temporal discriminability}. Reduced phase differences may make it harder for attention to distinguish adjacent frames, contributing to diminished motion and nearly static outputs.
    
    \item \textbf{Chromatic Collapse and Visual Artifacts:} Beyond the loss of motion, PI induces a breakdown of visual fidelity. The autoregressive rollout accumulates severe color saturation and structural artifacts (as seen in the final frames of Fig.~\ref{fig:pi_failure_analysis}).
\end{itemize}

\paragraph{YaRN: Color Oversaturation and Texture Drift}
The YaRN examples exhibit progressive color oversaturation, texture drift, and structural degradation. In the creature example, the fur and face develop pronounced color shifts, while repetitive textures emerge in the background. In the coral reef example, fish and coral structures become increasingly distorted as colors grow oversaturated. These changes accumulate over the rollout, progressively disrupting the original appearance and scene structure.

Together, these observations indicate that directly applying PI and YaRN does not provide stable length extrapolation in AR video generation, motivating modality-specific adaptation of LLM context-extension techniques.

\section{Proofs of Propositions \ref{prop:ans_properties} and \ref{prop:energy}}
\label{app:proofs}

\subsection{Distributional Properties and Energy Monotonicity}
\label{subsec:proof_energy}

\paragraph{Proof of Proposition~\ref{prop:ans_properties}}
Within a chunk of length $f$, $\mathbf{z}_0 \sim \mathcal{N}(\mathbf{0}, \mathbf{I}_d)$ and $\mathbf{z}_u = \rho \mathbf{z}_{u-1} + \sqrt{1-\rho^2}\boldsymbol{\epsilon}_u$ for $u \in \{1,\dots,f-1\}$, where $\boldsymbol{\epsilon}_u \sim \mathcal{N}(\mathbf{0}, \mathbf{I}_d)$ are mutually independent and independent of $\{\mathbf{z}_0,\dots,\mathbf{z}_{u-1}\}$, with $\rho \in [-1, 1]$. Since each $\mathbf{z}_u$ is a linear combination of independent zero-mean Gaussian vectors, $(\mathbf{z}_0,\dots,\mathbf{z}_{f-1})$ is jointly Gaussian with $\mathbb{E}[\mathbf{z}_u]=\mathbf{0}$. By induction on $u$, if $\operatorname{Cov}(\mathbf{z}_{u-1})=\mathbf{I}_d$, then
\begin{equation}
\begin{aligned}
    \operatorname{Cov}(\mathbf{z}_u)
    &= \rho^2 \operatorname{Cov}(\mathbf{z}_{u-1}) + (1-\rho^2)\operatorname{Cov}(\boldsymbol{\epsilon}_u) \\
    &= \rho^2 \mathbf{I}_d + (1-\rho^2)\mathbf{I}_d = \mathbf{I}_d,
\end{aligned}
\end{equation}
so $\mathbf{z}_u \sim \mathcal{N}(\mathbf{0}, \mathbf{I}_d)$ for all $u \in \{0,\dots,f-1\}$ (including $\rho=\pm 1$, where $\sqrt{1-\rho^2}=0$ and $\mathbf{z}_u = \rho^u \mathbf{z}_0$). Moreover, for $u > v$, unrolling Eq.~\eqref{eq:ans_recursion} gives $\mathbf{z}_u = \rho^{u-v}\mathbf{z}_v + \sum_{k=v+1}^{u}\rho^{u-k}\sqrt{1-\rho^2}\boldsymbol{\epsilon}_k$. Because $\boldsymbol{\epsilon}_k$ is independent of $\mathbf{z}_v$ for all $k > v$, we have $\operatorname{Cov}(\mathbf{z}_u, \mathbf{z}_v) = \rho^{u-v}\operatorname{Cov}(\mathbf{z}_v) = \rho^{u-v}\mathbf{I}_d$. By symmetry, $\operatorname{Cov}(\mathbf{z}_u, \mathbf{z}_v) = \rho^{|u-v|}\mathbf{I}_d$ for all $u,v \in \{0,\dots,f-1\}$, establishing the intra-chunk symmetric Toeplitz covariance structure while chunks are sampled independently across chunk boundaries.\hfill$\square$

\paragraph{Proof of Proposition~\ref{prop:energy}}
We define the adjacent-difference energy $\mathcal{E}_T(\rho)$ as the expectation of the squared $L_2$ norm of the first-order intra-chunk temporal differences:
\begin{equation}
\begin{aligned}
    \mathcal{E}_T(\rho)
    &\triangleq \mathbb{E} \left[ \sum_{u=1}^{f-1} \|\mathbf{z}_u - \mathbf{z}_{u-1}\|_2^2 \right] \\
    &= \sum_{u=1}^{f-1} \mathbb{E} \bigl[ \|\mathbf{z}_u\|_2^2 + \|\mathbf{z}_{u-1}\|_2^2 \\
    &\hspace{5em} - 2\mathbf{z}_u^\top \mathbf{z}_{u-1} \bigr].
\end{aligned}
\label{eq:energy_expansion}
\end{equation}

\paragraph{Marginal Variance} From Proposition~\ref{prop:ans_properties}, each frame $\mathbf{z}_u$ follows $\mathcal{N}(\mathbf{0}, \mathbf{I}_d)$, yielding
\begin{equation}
    \mathbb{E}\|\mathbf{z}_u\|_2^2 = \operatorname{Tr}(\operatorname{Cov}(\mathbf{z}_u)) = \operatorname{Tr}(\mathbf{I}_d) = d.
\label{eq:marginal_var}
\end{equation}

\paragraph{Cross-Correlation} From $\mathbf{z}_u = \rho \mathbf{z}_{u-1} + \sqrt{1-\rho^2} \boldsymbol{\epsilon}_u$, we compute the inner-product expectation:
\begin{align}
    \mathbb{E}[\mathbf{z}_u^\top \mathbf{z}_{u-1}] &= \mathbb{E}[(\rho \mathbf{z}_{u-1} + \sqrt{1-\rho^2} \boldsymbol{\epsilon}_u)^\top \mathbf{z}_{u-1}] \\
    &= \rho \mathbb{E}\|\mathbf{z}_{u-1}\|_2^2 + \sqrt{1-\rho^2} \mathbb{E}[\boldsymbol{\epsilon}_u^\top \mathbf{z}_{u-1}].
\end{align}
Since $\boldsymbol{\epsilon}_u$ is independent of $\mathbf{z}_{u-1}$, the second term vanishes, and substituting Eq.~\eqref{eq:marginal_var} gives
\begin{equation}
    \mathbb{E}[\mathbf{z}_u^\top \mathbf{z}_{u-1}] = \rho d.
\label{eq:cross_corr_res}
\end{equation}

\paragraph{Total Energy and Monotonicity} Substituting Eq.~\eqref{eq:marginal_var} and Eq.~\eqref{eq:cross_corr_res} into Eq.~\eqref{eq:energy_expansion} yields
\begin{equation}
    \mathcal{E}_T(\rho) = \sum_{u=1}^{f-1} (d + d - 2\rho d) = 2(f-1)(1-\rho)d.
\label{eq:final_energy_form}
\end{equation}
Differentiating with respect to $\rho$ gives
\begin{equation}
    \frac{d \mathcal{E}_T}{d \rho} = -2(f-1)d < 0
\end{equation}
for $f > 1$ and $d > 0$, proving strict monotonicity on $[-1, 1]$, with maximum
\begin{equation}
    \max_{\rho \in [-1, 1]} \mathcal{E}_T(\rho) = \mathcal{E}_T(-1) = 4(f-1)d
\end{equation}
and $\mathcal{E}_T(\rho) > 2(f-1)d = \mathcal{E}_T(0)$ throughout $\rho < 0$.\hfill$\square$

\subsection{Spectral Analysis and Motion Energy}
\label{subsec:proof_psd}

To interpret this energy shift in the frequency domain, consider a stationary AR(1) process with $|\rho| < 1$ and per-coordinate autocovariance $R(k) = \rho^{|k|}$. By the Wiener--Khinchin theorem, its Power Spectral Density (PSD) is the Discrete-Time Fourier Transform (DTFT) of $R(k)$:
\begin{equation}
    S_\rho(\omega) = \sum_{k=-\infty}^{\infty} \rho^{|k|} e^{-j\omega k} = 1 + \sum_{k=1}^{\infty} (\rho e^{-j\omega})^k + \sum_{k=1}^{\infty} (\rho e^{j\omega})^k.
\end{equation}
Using the geometric series $\sum_{k=1}^{\infty} r^k = \frac{r}{1-r}$ for $|r| = |\rho| < 1$, we obtain
\begin{equation}
    S_\rho(\omega) = 1 + \frac{\rho e^{-j\omega}}{1 - \rho e^{-j\omega}} + \frac{\rho e^{j\omega}}{1 - \rho e^{j\omega}}.
\end{equation}
Applying Euler's formula and simplifying gives
\begin{equation}
    S_\rho(\omega) = \frac{1 - \rho^2}{1 + \rho^2 - 2\rho \cos\omega}, \quad |\rho| < 1.
\label{eq:psd_derivation}
\end{equation}

\paragraph{High-Pass Characteristics and Antiphase Limit} 
For $\rho \in (-1, 0)$, evaluating the PSD at the Nyquist frequency $\omega = \pi$ and the DC component $\omega = 0$ yields
\begin{equation}
    S_\rho(\pi) = \frac{1-\rho}{1+\rho}, \quad S_\rho(0) = \frac{1+\rho}{1-\rho}.
\end{equation}
As $\rho \to -1^+$, $S_\rho(\pi) \to \infty$ while $S_\rho(0) \to 0$ (while preserving unit total variance $\frac{1}{2\pi}\int_{-\pi}^{\pi} S_\rho(\omega)\,d\omega = 1$), concentrating spectral mass toward $\omega = \pm\pi$. At $\rho = -1$, each chunk follows the alternating pattern $\mathbf{z}_u = (-1)^u \mathbf{z}_0$, maximizing the expected within-chunk adjacent-difference energy, as established in Proposition~\ref{prop:energy}.

\paragraph{Difference Energy via Parseval's Identity}
Passing the sequence through the first-order difference operator $\Delta \mathbf{z}_u = \mathbf{z}_u - \mathbf{z}_{u-1}$ corresponds to a linear filter with frequency response $H(\omega) = 1 - e^{-j\omega}$ and squared magnitude $|H(\omega)|^2 = 2(1 - \cos\omega)$. For $|\rho| < 1$, Parseval's relation gives the expected per-step, per-dimension difference energy:
\begin{equation}
\begin{aligned}
    \bar{\mathcal{E}}_T(\rho) &\triangleq \frac{1}{d} \mathbb{E} \|\mathbf{z}_u - \mathbf{z}_{u-1}\|_2^2 \\
    &= \frac{1}{2\pi} \int_{-\pi}^{\pi} |H(\omega)|^2 S_\rho(\omega)\, d\omega = 2(1 - \rho),
\end{aligned}
\end{equation}
matching $\mathcal{E}_T(\rho) = (f-1)d\,\bar{\mathcal{E}}_T(\rho) = 2(f-1)(1-\rho)d$ in Eq.~\eqref{eq:final_energy_form} and extending continuously to $\bar{\mathcal{E}}_T(-1) = 4$ as $\rho \to -1^+$. For $\rho < 0$, the peak gain of $|H(\omega)|^2$ aligns with the peak of $S_\rho(\omega)$ at $\omega = \pm \pi$, increasing the expected within-chunk difference energy relative to independent noise and providing a temporal variation prior for autoregressive denoising.

\section{Qualitative Evaluation of Attention Sink and YaRN}
\label{sec:attn_sink}
We compare Self Forcing without an attention sink, with an attention sink, and with YaRN plus an attention sink during 30s generation at $832\times480$ resolution. Both sink configurations retain the first three latent frames in the local attention window throughout inference. The YaRN configuration uses $\alpha=0.01$, $\beta=1.00$, and $\texttt{logit\_scale}=1.1$. ANS is disabled in all three configurations.

\begin{figure*}[t]
    \centering
    \includegraphics[width=0.8\linewidth]{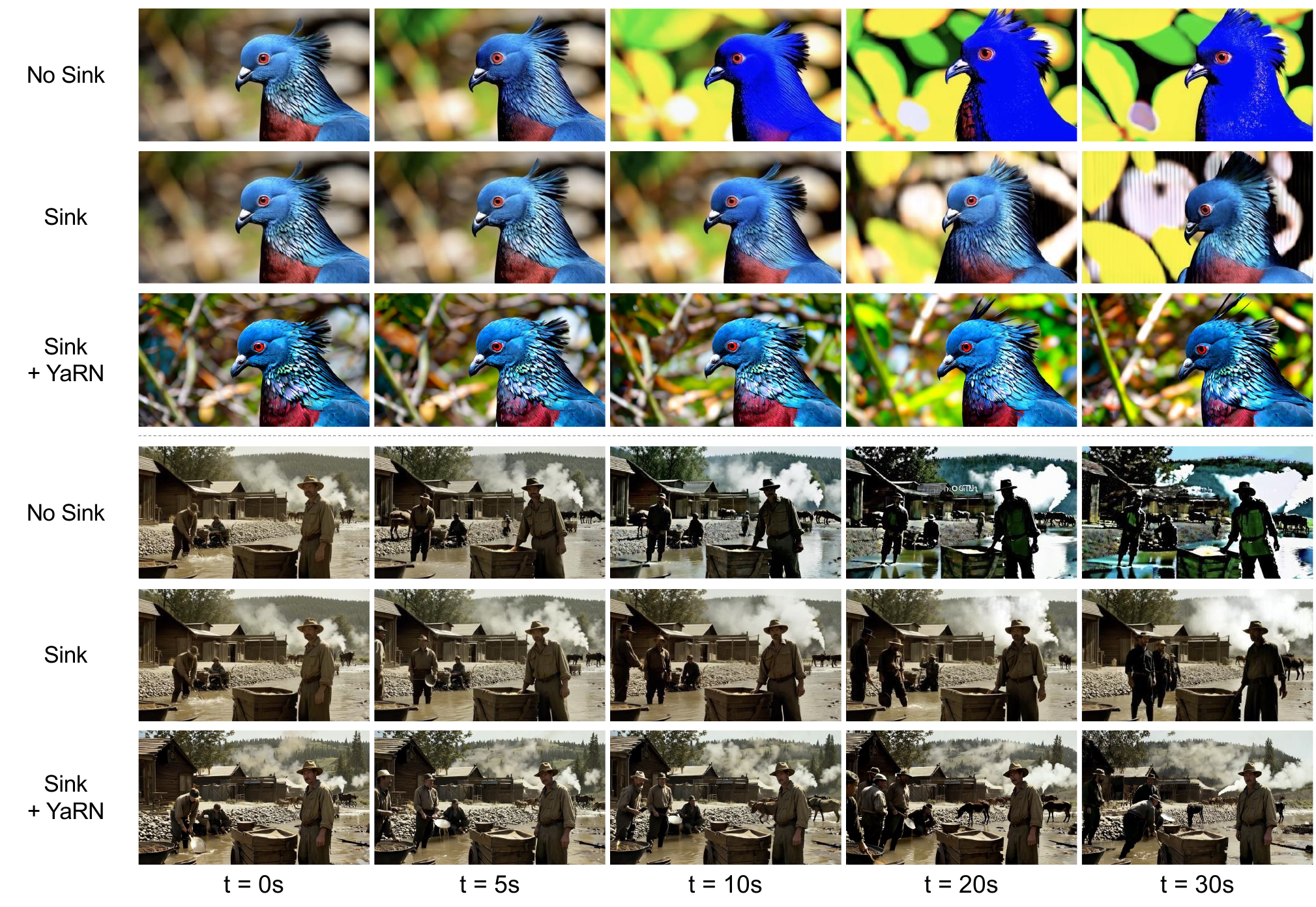}
    \caption{\textbf{Qualitative comparison of Self Forcing without attention sinks, with three initial latent frames retained as sinks, and with YaRN plus the same sink configuration.} Attention sinks mitigate severe visual degradation. With YaRN plus sinks, color shifts, oversaturation, and reduced texture detail are visible from the initial frames, while appearance and structural drift persist over time.}
    \label{fig:attn_sink}
\end{figure*}

\paragraph{Visual Consistency and Horizon Extension} As illustrated in Figure~\ref{fig:attn_sink}, the original Self Forcing baseline exhibits rapid color oversaturation and prominent visual artifacts as generation progresses. These phenomena indicate a clear limitation in the model's extrapolation capability beyond its 5-second training horizon. Upon introducing an attention sink ($L_{\text{sink}}=3$) during inference, these degradation issues are visibly mitigated, and frame-level structural stability is preserved over longer extrapolation lengths.

However, we also observe that an attention sink alone does not fully resolve quality degradation. In the first example, while the overall structure remains intact, the color and morphological features of the bird undergo noticeable shifts, accompanied by structural noise in the background. Similarly, in the second example, error accumulation persists even with sink tokens, leading to color bias and oversaturation. These observations suggest that while an attention sink serves as a vital semantic anchor, achieving long-range temporal consistency requires combining it with frequency-aware RoPE modulation and antiphase noise sampling, as discussed in Sections~\ref{sec:ntk} and~\ref{sec:ans}.

\paragraph{YaRN with Attention Sink}
We further explore combining attention sinks with an established LLM context-extension technique, YaRN. Figure~\ref{fig:attn_sink} shows that YaRN with attention sinks exhibits reduced visual fidelity from the initial frame ($t=0\text{s}$). Compared with the original Self Forcing output, the bird exhibits color shifts, oversaturation, and overly smooth feather textures with reduced fine-scale detail. This degradation may be related to YaRN's pre-softmax attention scaling, which modifies the pretrained attention distribution.

At the same time, we observe that temporal drift persists during autoregressive generation. Between 0s and 30s, the bird's crown morphology, plumage structure, and coloration change substantially, while the town scene develops a gray color cast and loses fine detail. These observations indicate that combining existing LLM context-extension techniques such as YaRN with attention sinks remains insufficient to maintain visual fidelity and temporal consistency during long-horizon generation.

\section{Detailed VBench-Long Benchmark Evaluation}
\label{app:vbench_long}

To provide a comprehensive assessment of our method's performance in long-horizon video generation, we conducted an evaluation across all 16 dimensions of the VBench-Long benchmark.

\paragraph{Experimental Setup}
Our evaluation uses the official VBench evaluation prompt set. For each prompt, we generate 30-second videos using 5 random seeds to ensure statistical robustness. Following the standard VBench pipeline, each generated video is segmented into clips and evaluated across all 16 dimensions, categorized into Quality Evaluation and Semantic Evaluation. We report the final Quality Score and Semantic Score aggregated from their sub-dimensions using the standardized normalization formulas and official weight coefficients provided by the VBench suite.

\paragraph{Full Quantitative Results}
Tables~\ref{tab:vbench_quality_full} and~\ref{tab:vbench_semantic_full} compare FLEX with four baselines. FLEX achieves the highest subject consistency, background consistency, temporal flickering, and aesthetic quality scores among the compared methods. It also leads in object class, spatial relationships, and overall consistency. Together, its Quality Score of 84.17 and Semantic Score of 80.73 yield the highest Total Score of 83.48.

To provide a clear view of our method's overall profile, Figure~\ref{fig:vbench_radar} presents a normalized radar chart comparing Self Forcing + FLEX against Self Forcing, LongLive, and Rolling Forcing across all 16 dimensions of the VBench-Long evaluation.

\begingroup
\makeatletter
\setlength{\@dblfptop}{0pt}
\setlength{\@fptop}{0pt}
\makeatother
\begin{table*}[p]
\begin{minipage}{\textwidth}
\captionsetup{aboveskip=4pt,belowskip=3pt}
\par\vspace{2pt}
{
\centering
    \captionof{table}{Full quality evaluation on VBench-Long.}
    \label{tab:vbench_quality_full}
    \setlength{\tabcolsep}{2pt} 
    \small
    \resizebox{0.85\textwidth}{!}{
    \begin{tabular}{l|ccccccc|c}
        \toprule
        \textbf{Model} & \makecell{Subject \\ Consistency} & \makecell{Background \\ Consistency} & \makecell{Temporal \\ Flickering} & \makecell{Motion \\ Smoothness} & \makecell{Aesthetic \\ Quality} & \makecell{Imaging \\ Quality} & \makecell{Dynamic \\ Degree} & \makecell{Quality \\ Score} \\ 
        \midrule
        Self Forcing & 98.58 & 97.06 & 98.93 & 98.56 & 61.80 & 68.77 & 29.38 & 81.94 \\
        Deep Forcing & 96.84 & 96.16 & 98.80 & 98.05 & 63.71 & 68.59 & \textbf{71.91} & \textbf{84.66} \\
        LongLive & 98.36 & 97.22 & 99.35 & \textbf{98.82} & 64.13 & 68.60 & 39.84 & 83.39 \\
        Rolling Forcing & 98.25 & 97.30 & 99.13 & 98.65 & 63.92 & \textbf{70.82} & 40.52 & 83.57 \\
        Self Forcing + FLEX & \textbf{98.70} & \textbf{97.94} & \textbf{99.59} & 98.53 & \textbf{65.43} & 69.11 & 44.32 & 84.17 \\
        \bottomrule
    \end{tabular}
    }\par
}
\par\vspace{4pt}
{
\centering
    \captionof{table}{Full semantic evaluation on VBench-Long.}
    \label{tab:vbench_semantic_full}
    \setlength{\tabcolsep}{2pt} 
    \small
    \resizebox{0.95\textwidth}{!}{
    \begin{tabular}{l|ccccccccc|c}
        \toprule
        \textbf{Model} & \makecell{Object \\ Class} & \makecell{Multiple \\ Objects} & \makecell{Human \\ Action} & Color & \makecell{Spatial \\ Relationship} & Scene & \makecell{Appearance \\ Style} & \makecell{Temporal \\ Style}  & \makecell{Overall \\ Consistency} & \makecell{Semantic \\ Score} \\ 
        \midrule
        Self Forcing & 90.76 & 81.35 & 94.70 & 79.38 & 72.74 & 50.46 & \textbf{21.50} & 22.91 & 25.23 & 76.42 \\
        Deep Forcing & 86.83 & 71.23 & 92.80 & 87.75 & 72.72 & 42.39 & 20.71 & 22.43 & 24.73 & 73.88 \\
        LongLive & 94.44 & \textbf{87.14} & \textbf{96.80} & \textbf{88.80} & 79.24 & 57.72 & 20.56 & \textbf{24.22} & 26.43 & \textbf{80.85} \\
        Rolling Forcing & 94.69 & 84.84 & 95.15 & 86.63 & 76.34 & \textbf{57.79} & 20.81 & 23.75 & 26.49 & 79.86 \\
        Self Forcing + FLEX & \textbf{94.94} & 85.95 & 95.45 & 88.39 & \textbf{82.04} & 56.41 & 20.51 & 24.09 & \textbf{26.67} & 80.73 \\   
        \bottomrule
    \end{tabular}
    }\par
}
\par\vspace{3pt}
\par\vspace{1pt}
{
\centering
    \includegraphics[width=0.70\textwidth]{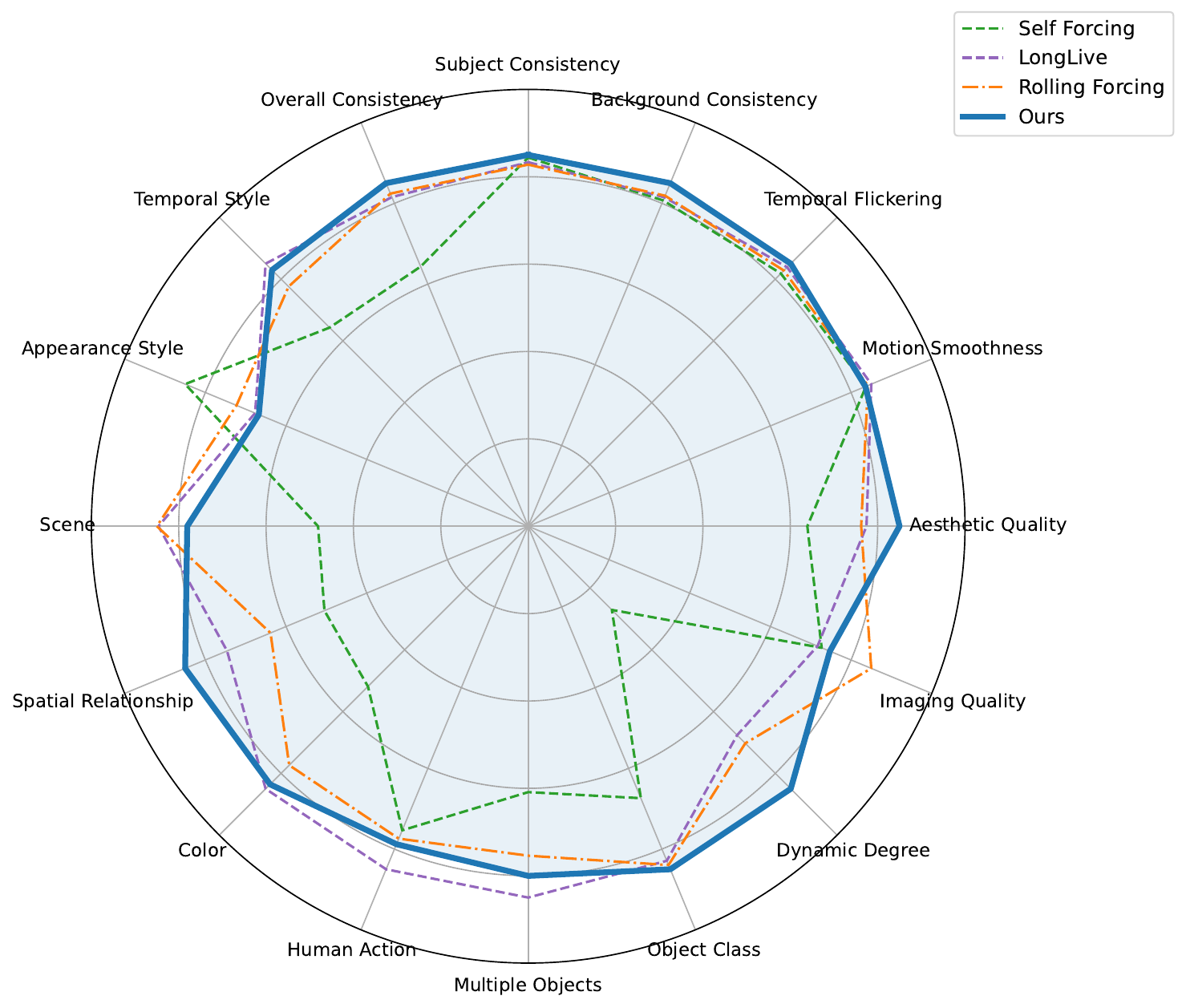}
    \captionof{figure}{\textbf{Comprehensive performance profile on VBench-Long.} The radar chart visualizes relative performance across 16 dimensions. Our method demonstrates a more balanced and expansive profile compared to existing baselines, indicating its robust capability across diverse long-horizon generation challenges.}
    \label{fig:vbench_radar}
}
\end{minipage}
\end{table*}
\clearpage
\endgroup

\twocolumn[{%
\begin{minipage}{\textwidth}
\captionsetup{hypcap=false,aboveskip=4pt,belowskip=4pt}
\centering
    \includegraphics[width=0.95\linewidth]{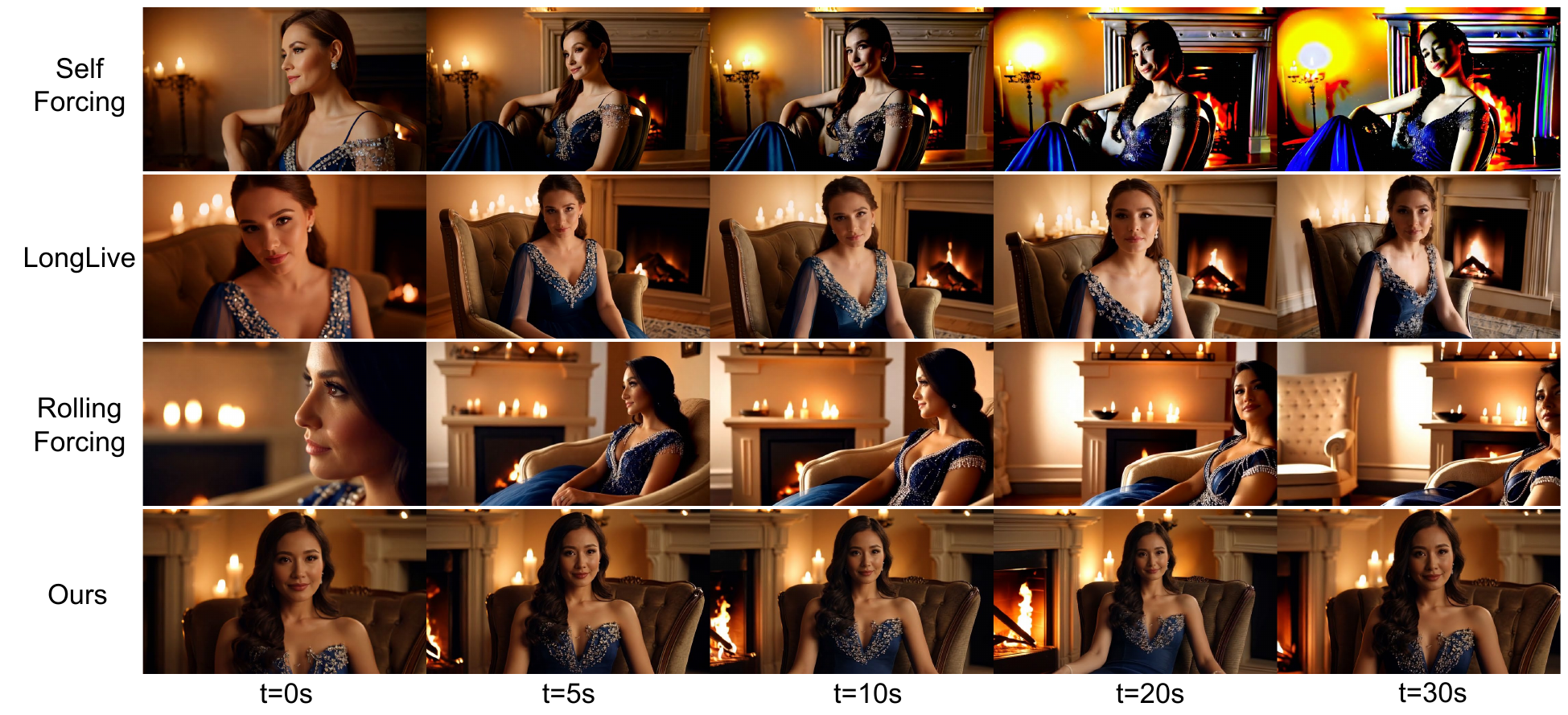}
    \captionof{figure}{\textbf{Human portrait.} Our method exhibits robust appearance consistency. Key features, including facial identity, clothing ornaments, and texture details, remain stable throughout the sequence, whereas baselines show progressive identity fading.}
    \label{fig:app_30s_1}
\par\vspace{0.6em}
    \includegraphics[width=0.95\linewidth]{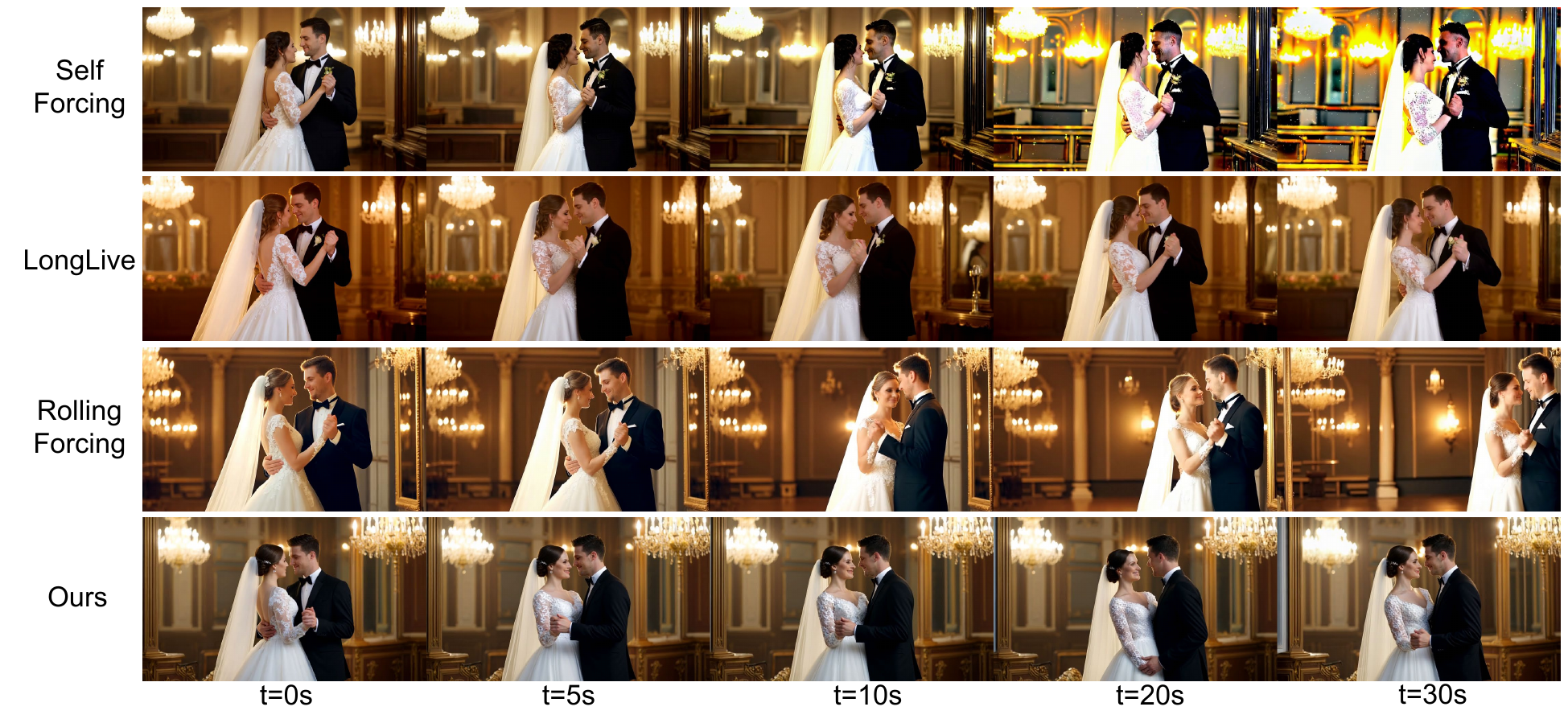}
    \captionof{figure}{\textbf{Scene.} Our method maintains better temporal consistency and visual quality over 30s. Notably, Rolling Forcing exhibits significant camera drift, with characters gradually moving toward the frame edge.}
    \label{fig:app_30s_2}
\vspace{0.8em}
\end{minipage}
}]

\setlength{\dbltextfloatsep}{12pt plus 2pt minus 2pt}
\begin{figure*}[!t]
\captionsetup{aboveskip=4pt,belowskip=4pt}
\centering
    \includegraphics[width=0.95\linewidth]{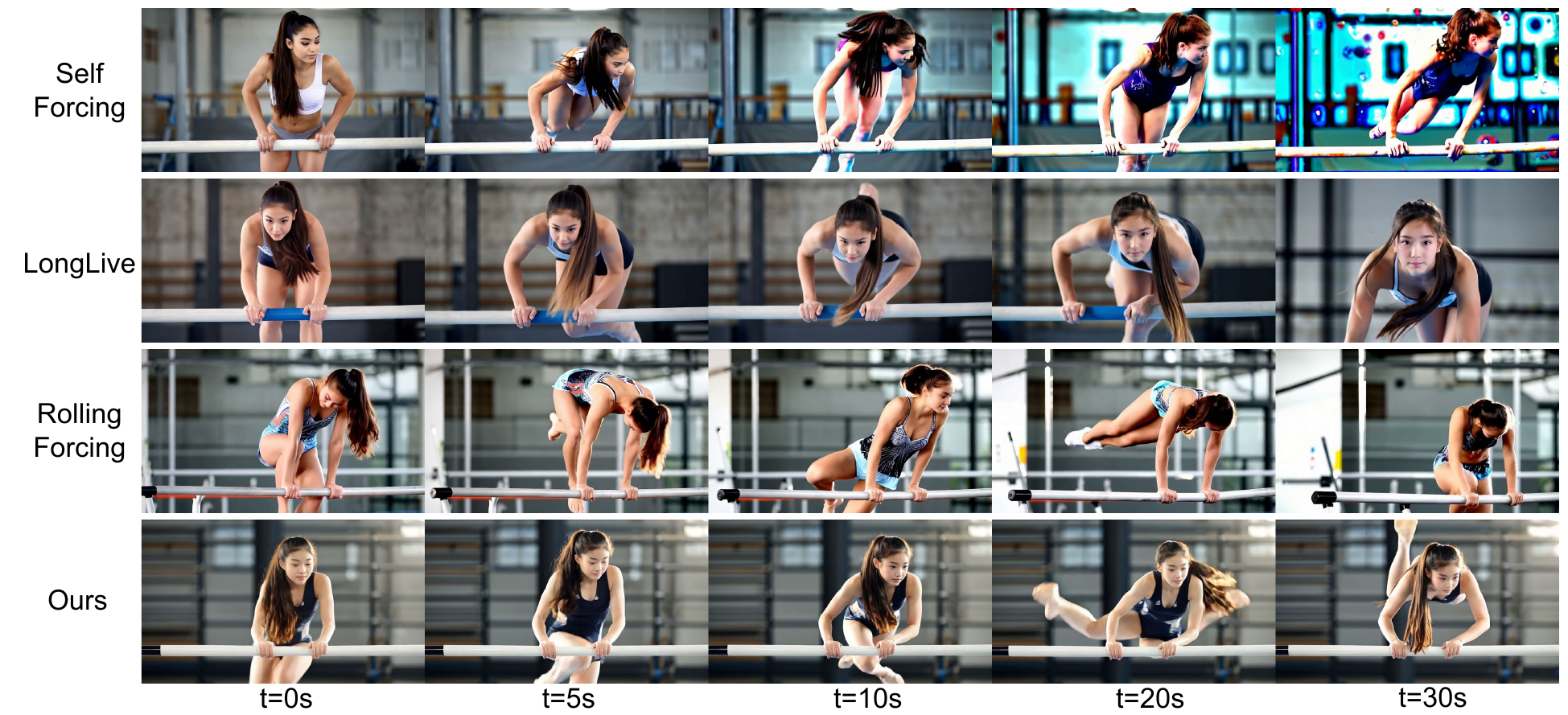}
    \captionof{figure}{\textbf{Sport \& Motion.} Our method preserves high-intensity motion without sacrificing identity stability. Hair and clothing details remain consistent in our results, whereas baselines suffer from progressive attribute drift as the generation horizon extends.}
    \label{fig:app_30s_3}
\par\vspace{0.6em}
    \includegraphics[width=0.95\linewidth]{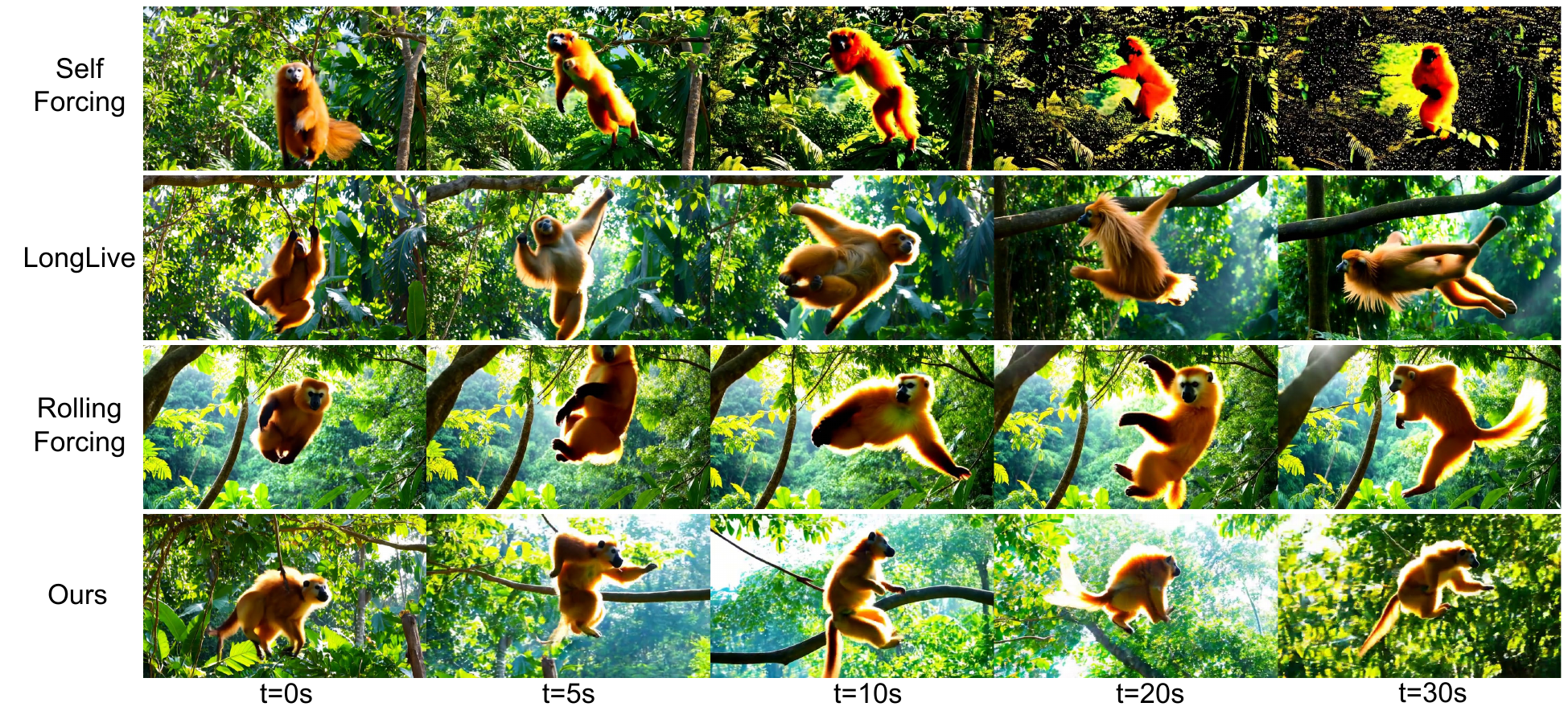}
    \captionof{figure}{\textbf{Animals \& Nature.} Our method preserves large-scale movements without appearance fading during the entire 30-second duration.}
    \label{fig:app_30s_4}
\end{figure*}

\begin{figure*}[!t]
\captionsetup{aboveskip=4pt,belowskip=4pt}
\centering
    \includegraphics[width=0.80\linewidth]{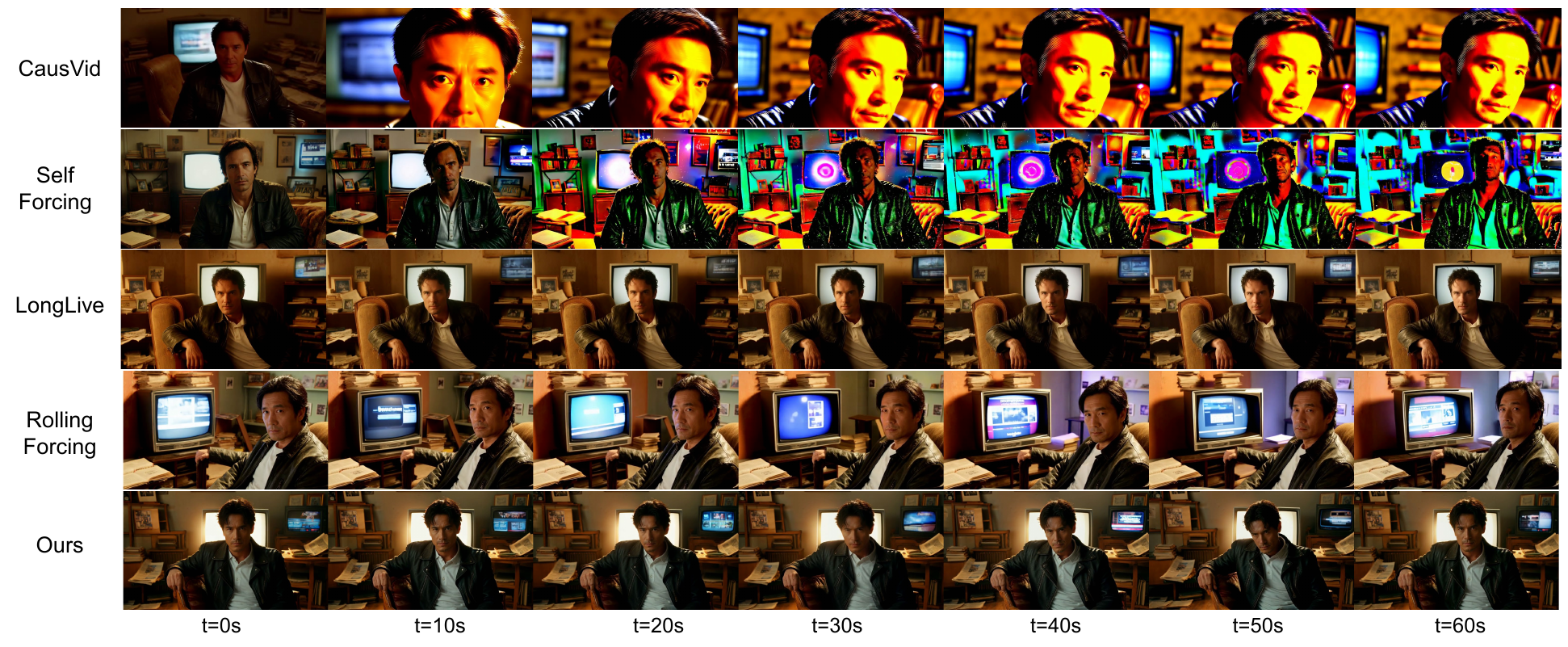}
    \captionof{figure}{\textbf{60s Character \& Scene.} In this extended indoor sequence, CausVid and Self Forcing gradually exhibit pronounced image collapse and color distortion. Rolling Forcing exhibits noticeable shifts in the background, color, and brightness. Our method preserves subject and background consistency from $t=0\text{s}$ to $t=60\text{s}$.}
    \label{fig:app_60s_1}
\par\vspace{0.6em}
    \includegraphics[width=0.80\linewidth]{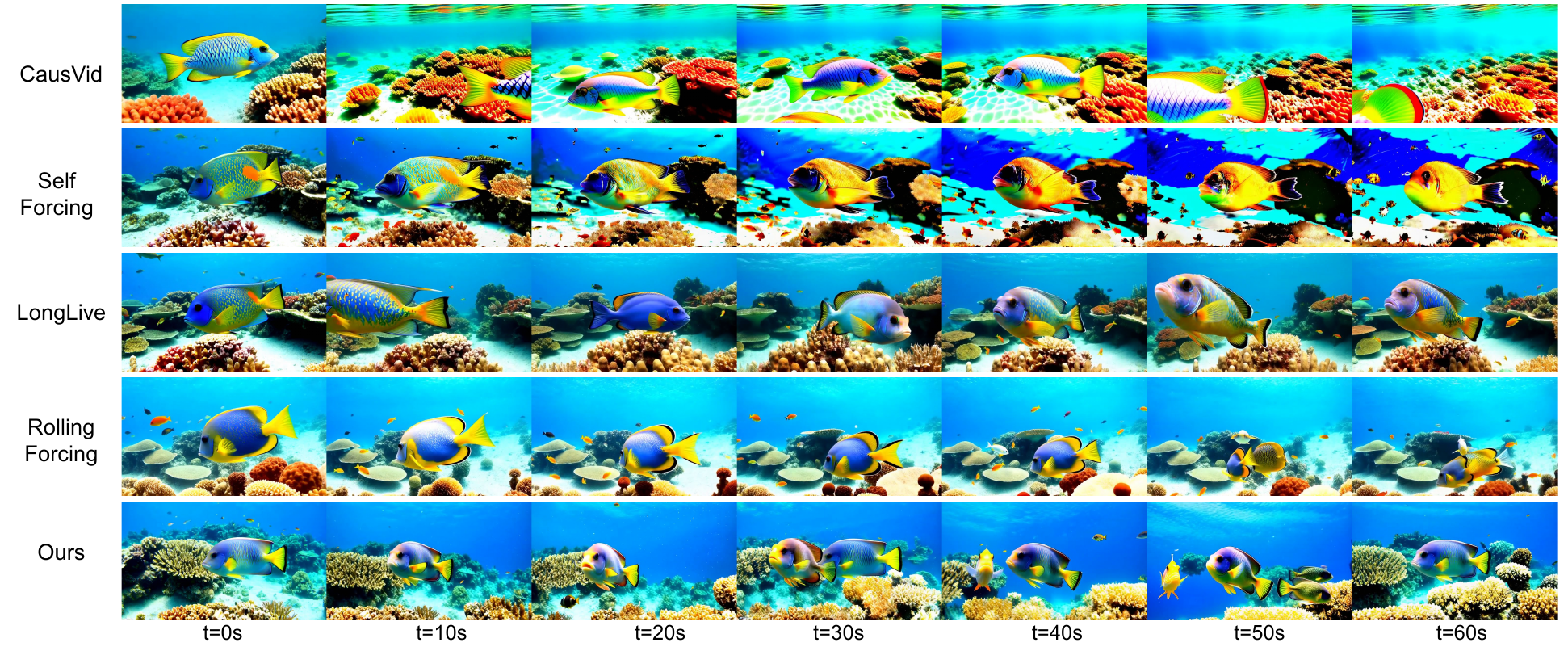}
    \captionof{figure}{\textbf{60s Marine World \& Dynamics.} Our approach avoids the static freeze in Self Forcing and identity fading seen in baselines such as LongLive, keeping the fish's species-specific appearance and the coral environment high-fidelity throughout the full minute.}
    \label{fig:app_60s_2}
\end{figure*}

\raggedbottom
\section{Additional Qualitative Comparisons}
\label{appendix:extended_qualitative}

As supplementary material to the qualitative analysis in Section~\ref{sec:Qualitative Results}, this section provides extended qualitative comparisons on 30-second and 60-second generation across diverse prompts.

\subsection{30-Second Video Generation Results}
We provide four additional 30-second video generation cases across distinct thematic categories. These results further demonstrate the generalization capability and temporal stability of our method compared with Self Forcing, LongLive, and Rolling Forcing at $t=0\text{s}, 5\text{s}, 10\text{s}, 20\text{s}, 30\text{s}$.

As illustrated in Figures~\ref{fig:app_30s_1}--\ref{fig:app_30s_4}, our method significantly enhances the extrapolation capability of the base Self Forcing model, preventing the rapid visual collapse typically observed beyond the 5s training length. Furthermore, our approach achieves competitive or superior performance across diverse scenarios compared with recent training-based methods such as LongLive, particularly in character consistency and motion dynamics.

\subsection{Long-Horizon 60-Second Synthesis}

We further evaluate our method on 60-second durations, a setting that typically leads to severe degradation in standard video generation paradigms due to accumulated error.

As demonstrated in Figures~\ref{fig:app_60s_1} and~\ref{fig:app_60s_2}, at an extended duration of 60s, our method achieves a stable enhancement in the extrapolation capability of Self Forcing, exhibiting performance comparable to training-based approaches such as LongLive that are specifically optimized on long-video samples.

\section{Implementation Details on LongLive Integration}
\label{sec:appendix_longlive}

We use the official LongLive inference pipeline for 240s generation. LongLive is trained on 60s sequences, so this experiment evaluates a fourfold extension of its training horizon.

\paragraph{Module integration}
We retain LongLive's native attention sink and apply frequency-aware RoPE modulation (NR) with $\alpha=1.0$ and $\beta=15.0$. We replace its independent noise sampling with ANS at $\rho=-1$. The configuration is fixed throughout the 240s rollout. Table~\ref{tab:longlive_comparison} and Figure~\ref{fig:longlive_240s} report the resulting quality and temporal behavior.

\section{Implementation Details on Krea-14B Integration}
\label{sec:appendix_krea}

\paragraph{Generation settings}
We use the Krea Realtime 14B checkpoint with 128 prompts and seeds 0 and 42 for each prompt. Both configurations use the same checkpoint, prompts, resolution of $832\times480$, and frame rate of 16 fps. We retain Krea-14B's three-frame latent chunks, five denoising steps, and timestep shift of 5.0 for the 30s experiment.

\paragraph{Module integration}
Krea-14B uses local temporal RoPE reindexing and recomputes the cached context. We retain this positional treatment and adapt FLEX through ANS and attention sink, avoiding a second positional transformation on top of Krea-14B's native scheme. ANS uses $(\mathbf{z},-\mathbf{z},\mathbf{z})$ within each chunk, with independent samples across chunks. The cache retains the first three latent frames as sink context alongside the three most recent frames. Both parts are recomputed using Krea-14B's native RoPE indexing. This integration preserves the existing positional handling while introducing FLEX's noise sampling and persistent visual context.

\end{document}